%% file: paper.tex
\documentclass{article}

\usepackage{bm}
\usepackage[table,xcdraw]{xcolor}
\usepackage{comment}
\usepackage{silence}
\usepackage{microtype}
\usepackage{graphicx}
\usepackage{subcaption}
\usepackage{booktabs} 

\usepackage{hyperref}

\usepackage[preprint]{icml2026}

\usepackage{amsmath}
\usepackage{amssymb}
\usepackage{mathtools}
\usepackage{amsthm}
\usepackage{bbm}
\usepackage{pifont}
\usepackage[capitalize,noabbrev]{cleveref}

\theoremstyle{plain}
\newtheorem{theorem}{Theorem}[section]

\newtheorem{lemma}[theorem]{Lemma}

\theoremstyle{definition}
\newtheorem{definition}[theorem]{Definition}
\newtheorem{assumption}[theorem]{Assumption}
\theoremstyle{remark}
\newtheorem{remark}[theorem]{Remark}

\usepackage[textsize=tiny]{todonotes}

\icmltitlerunning{Adapter-Augmented Bandits for Online Multi-Constrained Multi-Modal Inference Scheduling}

\begin{document}

\twocolumn[
  \icmltitle{Adapter-Augmented Bandits for Online\texorpdfstring{\\}{ } Multi-Constrained Multi-Modal Inference Scheduling}



  \icmlsetsymbol{equal}{*}

  \begin{icmlauthorlist}
    \icmlauthor{Xianzhi Zhang}{yyy,comp}
    \icmlauthor{Yue Xu}{yyy}
    \icmlauthor{Yinlin Zhu}{yyy}
    \icmlauthor{Di Wu}{yyy}
    \icmlauthor{Yipeng Zhou}{comp}
    \icmlauthor{Miao Hu}{yyy}
    \icmlauthor{Guocong Quan}{yyy}
  \end{icmlauthorlist}
  \icmlaffiliation{yyy}{School of Computer Science and Engineering, Sun Yat-sen University, Guangzhou, 510006, China.}
  \icmlaffiliation{comp}{School of Computing, Macquarie University, NSW 2109, Australia}
  \icmlcorrespondingauthor{Di Wu}{wudi27@mail.sysu.edu.cn}
  \icmlkeywords{Machine Learning, ICML}
  \vskip 0.3in
]



\printAffiliationsAndNotice{}  
\begin{abstract}

Multi-modal large language model (MLLM) inference scheduling enables strong response quality under practical and heterogeneous budgets, beyond what a homogeneous single-backend setting can offer.
Yet online MLLM task scheduling is nontrivial, as requests vary sharply in modality composition and latent reasoning difficulty, while execution backends incur distinct, time-varying costs due to system jitter and network variation.
These coupled uncertainties pose two core challenges: deriving semantically faithful yet scheduling-relevant multi-modal task representations, and making low-overhead online decisions over irreversible multi-dimensional budgets. 
Accordingly, we propose \emph{M$^2$-CMAB} (\underline{M}ulti-modal \underline{M}ulti-constraint \underline{C}ontextual \underline{M}ulti-\underline{A}rmed \underline{B}andit), a multi-adapter-enhanced MLLM inference scheduling framework with three components:
(i) a CLS-attentive, frozen-backbone \emph{Predictor} that extracts compact task representations and updates only lightweight adapters for action-specific estimation; 
(ii) a primal-dual \emph{Constrainer} that maintains online Lagrange multipliers to enforce long-horizon constraints via per-round objectives;
and (iii) a two-phase \emph{Scheduler} that balances exploration and exploitation under irreversible budgets. 
We establish a regret guarantee under multi-dimensional knapsack constraints.
On a composite multimodal benchmark with heterogeneous backends, \emph{M$^2$-CMAB} consistently outperforms state-of-the-art baselines across budget regimes, achieving up to 14.18\% higher reward and closely tracking an oracle-aided upper bound.
Codes are available at \url{https://anonymous.4open.science/r/M2CMAB/}.
\end{abstract}

\input{main_content}
\clearpage
\section*{Impact Statement}
This paper presents work whose goal is to advance the field of 
Machine Learning. There are many potential societal consequences 
of our work, none which we feel must be specifically highlighted here.
\bibliography{ref}
\bibliographystyle{icml2026}

\newpage
\appendix
\onecolumn
\input{appendix}

\end{document}

%% file: main_content.tex
\section{Introduction}

Multi-modal large language models (MLLMs) have achieved remarkable success in jointly reasoning over text, images, audio, and variable-modal inputs~\cite{Xu_2025_ICCV,Liu2025}. Advanced MLLM families, such as the Qwen-VL~\cite{bai2025qwen3vl,bai2025qwen25vl} and LLaVA series~\cite{li2024llava_onevision,li2024llava_next_interleave} with multiple variants, can ingest heterogeneous sensory inputs and generate richly grounded responses, fueling increasingly diverse modality compositions and service-level requirements~\cite{Hao2025,li2025arXiv}. In practice, cloud-scale deployment benefits from abundant compute and larger models, often delivering stronger response quality and robustness, whereas on-device deployment favors lightweight models that are cheaper to run and better suited for low-latency and privacy-sensitive operation. 
The collaboration of them offers a pragmatic paradigm, enabling real-world inference streams to achieve strong quality under practical cost constraints.

Despite this promise, resource-constrained deployment makes MLLM inference scheduling a fundamental problem~\cite{li2025CollaborativeSurvey}. 
Incoming requests vary in modality composition and reasoning difficulty, giving rise to diverse service-level requirements that span quality of responses and budget consumption~\cite{Xia2024,liyang2025arXiv}.
Meanwhile, execution backends exhibit substantially different performance-cost profiles, characterized by distinct model capacity, pricing, throughput, or network-dependent latency~\cite{yang2024PerLLMArXiv}.
As illustrated in Fig.~\ref{fig:distribution}, these coupled heterogeneities are reflected in average reward (response quality), latency, and monetary cost across on-device MLLMs and cloud APIs under a mixed multi-modal task trace.
We thus ask: \textit{How can we dynamically schedule MLLM inference across heterogeneous backends to maximize expected inference reward under stringent budgets?} Realizing this goal entails several core challenges.

The first bottleneck lies in representing inference tasks in both semantically faithful and scheduling-relevant.
On the task side, real-world requests differ in modality composition, input scale, and latent reasoning difficulty, which can be heavy-tailed even within a single domain~\cite{Yuan2025,yang2025moaoff}.
On the execution side, the observed inference results are confounded by stochastic decoding and system jitter, while end-to-end budget consumptions further entangle model compute with exogenous noise (e.g., network variation and time-varying on-device GPU contention), yielding non-stationary and drifting runtime patterns.
Some handcrafted features~\cite{Fu2024NIPS,Yuan2025} or coarse statistics (e.g., token-count surrogates~\cite{Fu2024NIPS}, modality indicators~\cite{Yuan2025,yang2025moaoff}, and static device profiles~\cite{yang2024PerLLMArXiv}), are brittle proxies of difficulty and cost, and can break under distribution or modality shifts~\cite{chen2025,Oh2025}. This motivates leveraging MLLMs themselves to produce task-level representations that summarize multi-modal semantics and implicit difficulty cues~\cite{Tu2025, Wu2025,Wang2025}, enabling more reliable online scheduling under uncertainty.

A second obstacle arises from the online scheduling itself.
First, online collaboration scheduling typically runs at the local and must stay on the critical path, leaving little headroom for expensive decision making~\cite{Ma2025}. Methods that either depend on auxiliary MLLM-based embedding extraction~\cite{Oh2025,Tu2025,Wu2025}, or require model tuning/adaptation before each decision~\cite{Huang2025,Ge2025}, can incur non-trivial computation and inference overhead~\cite{Zhou2025}.
Second, decisions are inherently \emph{irreversible} under strict budgets: resource slack is consumed on the fly, while future requests and their modality mix remain unknown. Consequently, heuristic or reinforcement-learning-style policies may over-explore and exhaust budgets early, yet offer limited guarantees on long-horizon constraint satisfaction~\cite{Pei2025,yang2025moaoff}.
These challenges call for a principled online learning framework that enables lightweight per-round decisions while explicitly controlling cumulative resource consumption~\cite{Jia2025, Agrawal2016,DBLP:conf/aistats/HanZWXZ23}.

\begin{figure}[t]
\centering      
\includegraphics[width=\columnwidth]{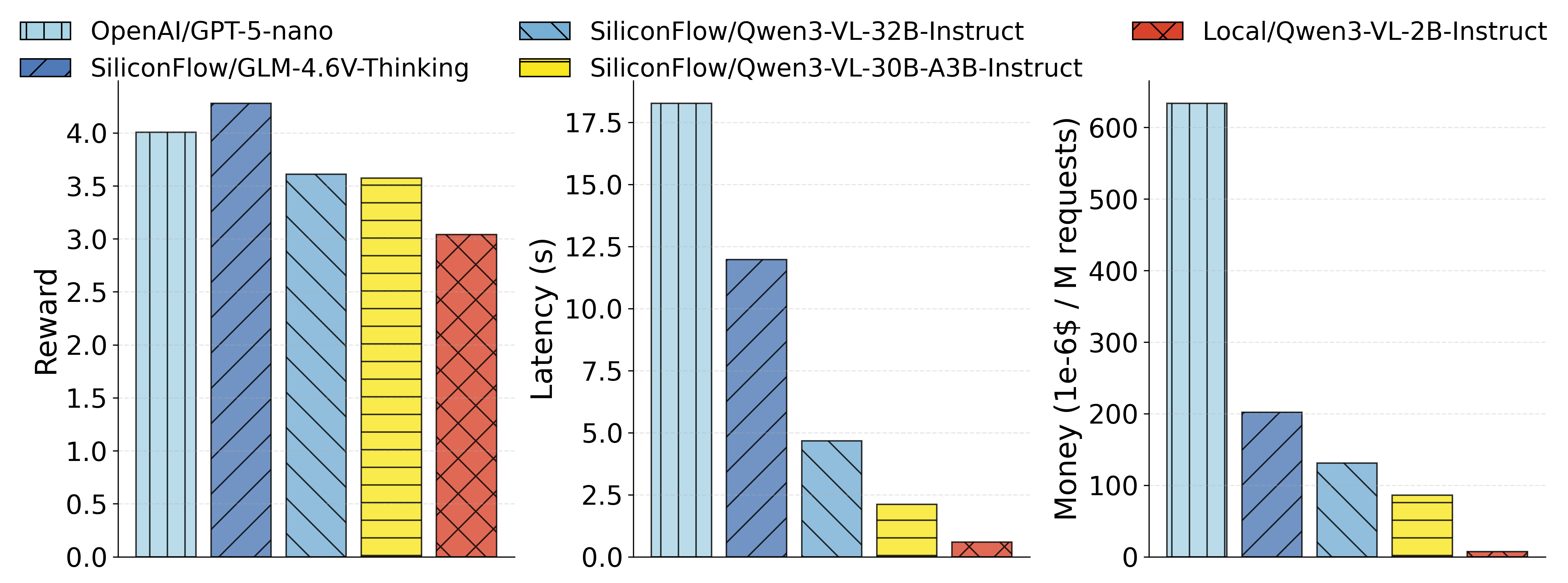}
\caption{Comparison of reward, latency, and monetary cost across MLLM inference backends on a mixed multi-modal task trace.}
\vspace{-6mm}
\label{fig:distribution}  
\end{figure}

To tackle the above challenges, we formulate resource-constrained MLLM inference scheduling as a {m}ulti-modal, {m}ulti-constraint {c}ontextual {b}andit {w}ith {k}napsacks (M$^2$-CBwK) problem. This formulation extends the classic CBwK framework~\cite{DBLP:conf/aistats/HanZWXZ23,Jia2025} by explicitly modeling structured multi-modal contexts and heterogeneous cost dimensions intrinsic to real MLLM deployments.
Based on this abstraction, we develop \emph{M$^2$-CMAB} (\underline{M}ulti-modal \underline{M}ulti-constraint \underline{C}ontextual \underline{M}ulti-\underline{A}rmed \underline{B}andit), an online MLLM task scheduling framework consisting of three components: a multi-adapter-enhanced \emph{Predictor}, an online primal-dual-based \emph{Constrainer}, and a two-phase \emph{Scheduler} for exploration-exploitation. We highlight the main contributions as follows.

\textbf{\ding{182} Efficient MLLM Representation.}
We extract a compact CLS-attentive task context from a frozen MLLM and update only lightweight reward/cost adapters online, enabling low-overhead, action-specific prediction while preserving generative capability of the local backend.
\textbf{\ding{183} Decoupled Long-Horizon Constraint Control.}
We maintain online Lagrange multipliers via a primal-dual update, decoupling cumulative constraints enforcement from per-round decisions under irreversible, multi-dimensional budgets.
\textbf{\ding{184} Efficient Contextual Scheduling.}
We propose a contextual exploration-exploitation policy that jointly uses multi-modal context and constraint feedback, and establish a regret guarantee under multiple knapsack constraints. 
\textbf{\ding{185} Realistic Benchmark and Strong Empirical Results.}
We build a benchmark with five backends, six datasets, and seven methods, and show \emph{M$^2$-CMAB} outperforms other state-of-the-art scheduling methods, achieving up to 14.18\% higher average reward, and most closely matches an oracle-aided counterpart that uses perfect per-round observed information.

\begin{figure*}[t]      
\centering      
\includegraphics[width=\linewidth]{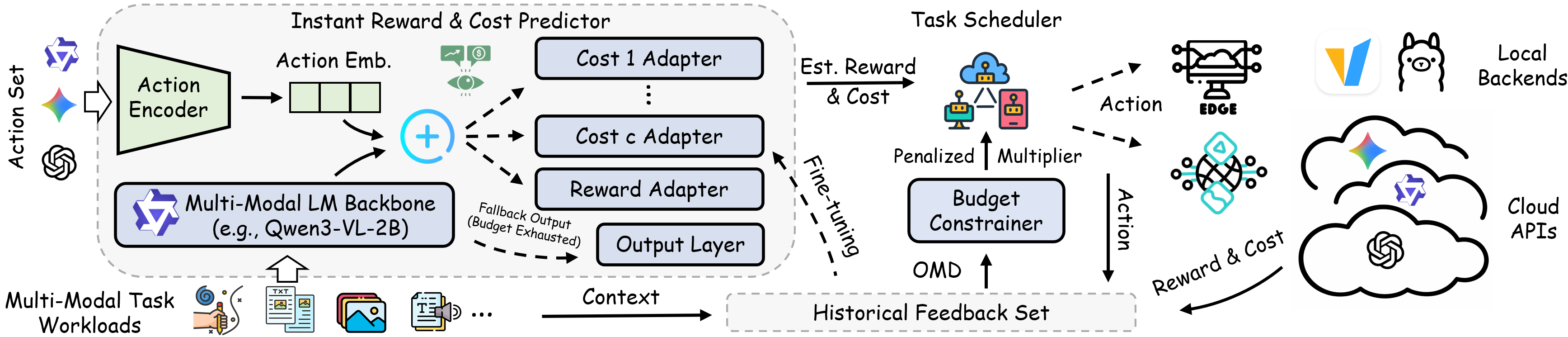}   
\caption{Overview of the \emph{M$^2$-CMAB} decision-making pipeline.}
\vspace{-5mm}
\label{fig:M$^2$-CMAB}  
\end{figure*}

\section{Preliminary and Problem Formulation}\label{sec:preliminaries}

Our study focuses on designing an online scheduling algorithm to address dynamic decision-making problems for MLLM inference tasks. At each round $t \in \{1,\ldots,T\}$, the scheduling algorithm receives an incoming task with the multi-modal context $\bm{x}^t = \{\bm{x}^t_{\omega}\}_{\omega \in \Omega}$, where $\Omega$ denotes the set of available modalities (e.g., text, vision, audio, and auxiliary metadata).
For each modality $\omega$, $\bm{x}^t_{\omega} \in \mathbb{R}^{d_{\omega}^{\text{in}} \times d_{\omega}^{\text{tok}}}$ represents its token-level embedding, where $d_{\omega}^{\text{tok}}$ denotes the number of tokens and $d_{\omega}^{\text{in}}$ denotes the per-token embedding dimensionality for modality $\omega$, jointly forming the task-level multi-modal prompt input to the MLLM. 

Given the context $\bm{x}^t$, the scheduler selects an action $a^t \in \mathcal{A}$ from the action space (e.g., a specific backend combining the execution environment, model choice, and reasoning mode) according to a learnable policy $\pi(\cdot \mid \bm{x}^t)$. The objective of the policy is to maximize the overall quality of service (QoS), quantified as the expected cumulative reward $\mathbb{E}\!\left[\sum_{t=1}^{T} r^t_{a^t}\right]$, while satisfying multiple long-term budget constraints $\bm{\Phi} \in \mathbb{R}^C_+$ across $C$ resource dimensions (e.g., monetary cost and latency).
Accordingly, the problem is formulated as a multi-modal, multi-constraint contextual bandit with knapsacks (M$^2$CBwK) problem, denoted by $\mathbb{P}1$, where $\mathcal{R}(\cdot,\cdot)$ denotes the reward function induced by the selected action and the observed context, $\bm{\varphi}^t_a = \bm{\mathcal{C}}(a,\bm{x}^t)$ represents the corresponding resource consumption vector in round $t$, and $\mathbbm{1}$ represents the all-ones vector. 
The constraints in~\eqref{EQ:c1} are applied component-wise and are dimensionless. Finally, $\Omega$ denotes the set of supported modalities that together constitute the multi-modal context $\bm{x}^t$.
\begin{subequations} \begin{align} 
\mathbb{P}1: \max_{\pi} &\;\mathbb{E}\!\left[\sum_{t=1}^T r^t_{a}\right], \label{eq:md_obj} \\
\text{s.t.} \quad &\sum_{t=1}^{T} \frac{\bm{\varphi}^t_{a}}{\bm{\Phi}} \;\preceq\; \mathbbm{1}, \label{EQ:c1}\\ 
& r^t_{a} := \mathcal{R}\left(a,\bm{x}^t\right),\; \bm{\varphi}^t_{a} := \bm{\mathcal{C}}\left(a,\bm{x}^t\right), \label{EQ:c2}\\ 
& a \sim \pi(\cdot \mid \bm{x}^{t}),\\
& \bm{x}^{t}=\{\bm{x}^{t}_{\omega}\}_{\omega\in\Omega},\; \forall t. \label{EQ:c3}
\end{align} 
\end{subequations}
The online task offloading problem $\mathbb{P}1$ captures the intrinsic dynamics of sequential multi-modal inference tasks under uncertainty. In this setting, MLLM inference outcomes are inherently stochastic: for a given context $\bm{x}^t$ and action $a^t$, the realized reward and cost depend on complex task-backend interactions and are only revealed after execution. Meanwhile, offloading decisions must be made irrevocably without access to future task arrivals, requiring the scheduler to adapt online while jointly managing multiple limited macro-budgets so that costly resources are reserved for sufficiently complex tasks. These characteristics give rise to two core challenges. \textbf{C1 (Uncertain payoff modeling)} arises from the difficulty of accurately characterizing $\mathcal{R}\left(a^t,\bm{x}^t\right)$ and $\bm{\mathcal{C}}\left(a^t,\bm{x}^t\right)$, since reward and cost distributions induced by MLLM inference admit no closed-form expressions. \textbf{C2 (Non-myopic budget coupling)} stems from the fact that greedy per-round optimization may prematurely exhaust unrecoverable budgets, resulting in suboptimal long-term QoS performance. Additionally, we summarize the key notations in Table~\ref{Table:Major Notations} in Appendix~\ref{appendix: Major Notations}.

\section{M$^2$-CMAB Design}\label{sec:algorithm design}

Our \emph{M$^2$-CMAB} framework to solve $\mathbb{P}1$, illustrated in Fig.~\ref{fig:M$^2$-CMAB}, consists of three coupled components. The \textit{Predictor} uses an adapter-based MLLM to estimate the reward $r^{t}$ and cost vector $\bm{\varphi}^t$ from the given action and context, addressing uncertain payoff modeling (\textbf{C1}). The \textit{Constrainer} then applies budget-aware penalties to regulate per-round resource usage, mitigating non-myopic budget coupling (\textbf{C2}). Finally, the \textit{Scheduler} combines the predicted rewards and penalized costs to make the exploration-exploitation decision.

\subsection{Reward and Cost Predictor Design}\label{sec:Reward Prediction}
Accurate reward and cost prediction is central to efficient online decision-making in \emph{M$^2$-CMAB}, yet extracting task-aligned embeddings from MLLMs for scalar regression is non-trivial. 
A straightforward approach is to fine-tune the backbone MLLM model~\cite{Ge2025, Huang2025}. While effective for certain downstream tasks, directly fine-tuning backbone models for reward and cost regression is often mismatched to online decision-making settings. 
First, trajectory-level feedback in online scheduling is typically noisy and non-stationary, making fine-tuned embeddings prone to overfitting transient patterns and degrading representation stability over time. 
Second, repeated fine-tuning incurs substantial computational and system overhead, rendering it impractical for scenarios that require fast adaptation training and frequent prediction.

Freezing the backbone alleviates these issues but introduces an additional challenge. In frozen-backbone settings, directly extracting embeddings from the final output token remains suboptimal, as its hidden state is primarily optimized for next-token prediction and does not reliably summarize global task semantics. This mismatch becomes more pronounced in long or multi-modal contexts, where scalar reward and cost regression depends critically on consistent task-level representations rather than generative diversity.

To address these challenges, we adopt a CLS-attentive multi-adapter-based prediction pipeline built on one frozen backbone. Specifically, we freeze the parameters $\bm{\vartheta}$ of MLLM backbone $f_{\bm{\vartheta}}(\cdot)$ to preserve its generative capability and representation consistency. To avoid directly using the final token output, an explicit \texttt{[CLS]} token is introduced as a global semantic anchor at the beginning of the multi-modal context input tokens $\bm{x}^t$. Its multi-head self-attention weights $\bm{\alpha}_h^t = \bm{\mathrm{I}}_{h}^t[\mathrm{CLS},:]$ are then used to perform attention-based pooling over the hidden states $\bm{\mathrm{H}}^t$, where $\bm{\mathrm{H}}^t \in \mathbb{R}^{L \times d_{\mathrm{hid}}}$ and $\bm{\mathrm{I}}^t \in \mathbb{R}^{H \times L \times L}$ are extracted from the last output layer of the frozen backbone, thereby yielding a task hidden representation $\bm{z}_x^t = f_{\bm{\vartheta}}(\bm{x}^t)$ aligned with context semantics. 

Given the task representation $\bm{z}_x^t$, prediction is performed by conditioning on both the contextual information and the selected action. To this end, we concatenate $\bm{z}_x^t$ with the action embedding $\bm{z}_a$, obtained from a language modality embedder, and feed the resulting representation into a set of lightweight adapters $\mathcal{F}_{\bm{\theta}}(\bm{z}_a || \bm{z}_x^t)$, parameterized by $\bm{\theta}_r^t$ and $\{\bm{\theta}_\varphi^{t,(c)}\}_{c=1}^C$ for reward and per-dimension constraint estimation, respectively. These adapters serve as task-specific heads that map the shared representation to scalar reward and cost estimates, enabling efficient and stable prediction without modifying the frozen backbone parameters.

Specifically, the reward adapter can be formulated
$
    \hat{r}^t_a = \hat{\mathcal{R}} \!\left(a,\bm{x}^t\, ;\, \bm{\Theta}^t_r\right),
$
where $\bm{\Theta}^t_r=\{\bm{\vartheta},\bm{\theta}^t_r\}$
includes the frozen MLLM backbone parameters and the trainable adapter
parameters for reward regression. Likewise, for $C$
budget-constraint dimensions (e.g., monetary cost, latency, energy), we
assign each dimension $c$ an independent cost adapter
$
    \hat{\varphi}^{t,(c)}_{a} 
    = \hat{\mathcal{C}}\!\left(a,\bm{x}^t\, ;\, \bm{\Theta}^{t,{(c)}}_\varphi \right),
$
producing the cost vector
$\hat{\bm{\varphi}}^t_a$ that
characterizes the estimated per-round resource consumption for action
$a$ with parameter $\bm{\Theta}^{t,{(c)}}_\varphi=\{\bm{\vartheta},\bm{\theta}^{t,(c)}_\varphi\}$.

To train these predictors, we adopt a unified regularized least-squares objective. Let $\mathcal{F}_{\bm{\theta}}$ denote either the reward predictor or a cost predictor for dimension $c$, and let $y^\tau$ represent its corresponding supervision signal, i.e., reward $r^\tau$ or cost
$\varphi^{\tau,(c)}$. The adapter parameters $\bm{\theta}$ are updated by
\begin{equation}\label{eq:generic_reg_obj_reordered}
    \min_{\bm{\theta}} \mathcal{J}(\bm{\theta})
    = \frac{1}{2|\Upsilon^t|}
      \sum_{\tau < t}\!
      \bigl(\mathcal{F}_{\bm{\theta}}^\tau - y^\tau\bigr)^{2}
      + \frac{\eta}{2}\,\|\bm{\theta}-\bm{\theta}^0\|_2^{2},
\end{equation}
where $\Upsilon^t= \{(r^\tau, \bm{\varphi}^\tau, a^\tau, \bm{x}^\tau) \mid \tau < t, \tau\in\mathbb{N}_+\}$ is the historical set of observations before round $t$, $\bm{\theta}^0$ is the initialization of the adapter parameters, and $\eta>0$ is the regularization coefficient. Instantiating
\eqref{eq:generic_reg_obj_reordered} with $(\mathcal{F}_{\bm{\theta}},y^\tau)$
equal to $(\hat{r}_{a}^\tau,r^\tau)$ gives the reward regression loss
$\mathcal{J}_r(\bm{\theta}^t_r)$, while choosing
$(\hat{\varphi}_{a}^{\tau,(c)},\varphi^{\tau,(c)})$ gives the cost loss
$\mathcal{J}^{(c)}_\varphi(\bm{\theta}^{t,(c)}_\varphi)$ for each constraint $c$. Implementation details of the reward and cost predictors are deferred to Appendix~\ref{appendix:adapter_details}.

\subsection{Long-Term Constrainer Design}\label{sec:Online Dual Variables Update}
To address challenge \textbf{C2}, we introduce a \textit{Constrainer} that decouples time-dependent budget constraints via an online queue-based optimization framework~\cite{ouyang2023dynamic, Zhang2025TSC}, enabling per-round budget allocation in $\mathbb{P}1$. 
Specifically, the long-term budget constraint in Eq.~\eqref{EQ:c1} can be reformulated as a queue-recursion process by associating each constraint with a Lagrange multiplier $\bm{\lambda}^{t}\in\mathbb{R}^C_+$ at round $t$. 
Under this formulation, the cumulative objective in Eq.~\eqref{eq:md_obj} admits a Lagrangian form 
$\mathcal{L}(\bm{a}, \bm{\lambda}) 
= \sum_{t=1}^T \mathcal{L}^{t}\,(a ,\bm{\lambda}^{t}), 
$
where the per-round Lagrangian function is given by
\begin{equation}
\begin{split}
      &\mathcal{L}^{t}\,(\,a,\bm{\lambda}^{t})=
       - r^{t}_{a} +  \langle\frac{\bm{\varphi}^{t}_a}{\bm{\Phi}} - \frac{1}{T},\bm{\lambda}^{t}\rangle.
\end{split}
\end{equation}
Here, $\bm{\Phi}\in\mathbb{R}^C_+$ denotes the total budget vector, $\bm{\varphi}^{t}_a\in\mathbb{R}^C_+$ represents the per-round resource consumption, and $\langle\cdot,\cdot\rangle$ is the standard inner product.
This reformulation converts $\mathbb{P}1$ into a saddle-point problem that maximizes reward while penalizing budget violations,
$
\min_{\bm{a}}
\max_{\bm{\lambda}}\mathcal{L}(\bm{a}, \bm{\lambda})\; \text{s.t.}\; \eqref{EQ:c2}-\eqref{EQ:c3}.
$ 
Importantly, the use of Lagrange multipliers enables temporal decoupling, allowing the optimization to be carried out on a per-round basis by focusing on $\mathcal{L}^{t}(a^{t},\bm{\lambda}^{t})$.

To update the dual variables online, we apply the online mirror descent (OMD) algorithm~\cite{DBLP:journals/ftopt/Hazan16,Shai2012}, which performs iterative updates over the dual domain by minimizing the per-round objective:
\begin{equation}
\label{con:dual_problem}
\mathcal{D}^{t}(\bm{\lambda}^{t}) = \max_{a}   \left(r^{t}_a - \langle\frac{\bm{\varphi}^{t}_a}{\bm{\Phi}} -\frac{1}{T},\bm{\lambda}^{t}\rangle\right).
\end{equation}
The dual variable $\bm{\lambda}^{t}$ is constrained to the convex set
$\mathcal{V} = \{\bm{\lambda}\in\mathbb{R}^C_+:\|\bm{\lambda}\|_1 \leq \Lambda\}$,
where $\Lambda>0$ specifies the radius of the feasible region.
With the feasible set $\mathcal{V}$ determined, the \textit{Constrainer} updates the dual variable $\bm{\lambda}^{t+1}$ at each subsequent round according to
\begin{equation}
    \label{eq:lambdaupdate}
     \bm{\lambda}^{t+1}=\arg \min _{\bm{\lambda} \in \mathcal{V}}\left(\langle\nabla \mathcal{D}^{t}\left(\bm{\lambda}^{t}\right), \bm{\lambda}\rangle+\frac{\mathcal{B}\left(\bm{\lambda}, \bm{\lambda}^{t}\right)}{\varrho_t} \right), 
\end{equation}
where $\mathcal{B}(\cdot,\cdot)$ denotes the Bregman divergence induced by the negative-entropy generating function~\cite{DBLP:conf/aistats/HanZWXZ23}, and $\varrho_t$ is the step size of the OMD. This update corresponds to a mirror-descent step that balances the instantaneous dual gradient and proximity to the previous iterate.

Overall, the \emph{Constrainer} incorporates long-term budget constraints into per-round reward maximization by dynamically adjusting the dual variables $\bm{\lambda}^t$, thereby imposing adaptive and forward-looking penalties on budget consumption. 

\subsection{Task Scheduler Design}\label{Pirvacy Budget Allocation}

Leveraging prediction results and decoupled constraints, the final component of \emph{M$^2$-CMAB} assigns each action a scalar score, upon which a two-phase Contextual Multi-Armed Bandit (CMAB) framework is employed to balance the exploration-exploitation trade-off and make final online scheduling decisions under uncertainty.

%
\subsubsection{\textbf{Initial Phase}} The target of the {initial} phase is to  establish the space $\mathcal{V} = \left\{\bm{\lambda} \in \mathbb{R}^C_+:\,\|\bm{\lambda}\|_1 \leq \Lambda\right\}$ for \emph{Constrainer}, so that 
the  budget consumption can be properly penalized  in the exploration-exploitation phase.
Inspired by~\citet{DBLP:conf/aistats/HanZWXZ23}, we employ a regression-based procedure to estimate $\lambda$. The {initial} phase lasts $(A+1)\cdot T_0$ rounds to evaluate $\Lambda$, where $A$ denotes the size of the action space and $T_0$ is a hyper-parameter specifying the number of times each action is selected.
To ensure that the OMD algorithm achieves $\text{O}(\sqrt{T})$ regret bounds, our preferred choice is to set $\Lambda = T \cdot \frac{\text{OPT}}{\Phi_\text{min}}$(Lemma~\ref{lemma: OPT_Z}), where OPT is defined in Definition~\ref{def:OPT} and $\Phi_\text{min} = \|\bm{\Phi}^{-1}\|_\infty$. 
However, OPT cannot be solved accurately in an online scenario.
Therefore, we employ a regression algorithm in the {initial phase} including $(A + 1)\cdot T_0$ rounds to evaluate $\Lambda$.
Generally speaking, there are two steps in the {initial phase}. 

\textbf{Stage \ding{182}:} For the first $A\cdot T_0$ rounds, we play different $A$ actions in CMAB evenly for $T_0$ times to gather historical records set $\Upsilon^{A\cdot T_0} = \{(r^\tau, \bm{\varphi}^\tau, a^\tau, \bm{x}^\tau) \mid \tau \leq A\cdot T_0, \tau\in\mathbb{N}_+\}$.
With the historical records, we can train an adapter as an oracle to predict the reward and cost $\hat{r}_a^t,\,\forall a\in\mathcal{A},\,\forall t\in\mathcal{T}_0$, where $\mathcal{T}_0=\left\{t: A\cdot T_0+1 \leq t \leq (A+1) \cdot T_0\right\}$. 

\textbf{Stage \ding{183}:} For the latter total $T_0$ rounds, \emph{i.e.}, we select actions $a^t$ randomly and observe the reward $r^t$, where $t\in\mathcal{T}_0$. 
Then, we use the observed history s prediction model to estimate $\hat{\text{OPT}}\left(T_0\right)$ by solving the linear programming problem~\eqref{eq:linear_programming}, where $\bm{o}$ is the solution matrix.
Additionally, $T$ denotes the maximum number of training rounds, $C$ is the total number of constraints, $\mathbbm{1}$ represents the all-ones vector, and $\bm{\Phi}$ denotes the budget vector across all constraints. 
The term $\mathcal{M}(T_0)$ characterizes the aggregated estimation uncertainty arising from both the reward and cost regression oracles. 
Once all required parameters are obtained, the estimated value $\hat{\mathrm{OPT}}(T_0)$ can be computed by solving the linear program in Eq.~\eqref{eq:linear_programming} using a standard linear programming solver.
\begin{subequations}\label{eq:linear_programming} 
\begin{align}  
\hat{\text{OPT}}\left(T_0\right)&=\max_{\bm{o} \in\mathcal{O}} \frac{1}{T_0} \sum_{t \in \mathcal{T}_0} \sum_{a \in\mathcal{A}}  o_{a}^t\, \hat{\mathcal{R}}\!\left(a,\bm{x}^t\, ;\, \bm{\Theta}^t_r\right), \label{eq: linear_programming objective}\\ 
\text{s.t.}\quad \sum_{t \in \mathcal{T}_0}& \sum_{a\in\mathcal{A}} o_{a}^t\,\frac{\hat{\bm{\mathcal{C}}}\!\left(a,\bm{x}^t\, ;\, \bm{\Theta}^{t}_\varphi\right)}{T_0\bm{\Phi}}\preceq \frac{\mathbbm{1}}{T}+2 \frac{\mathcal{M}(T_0)}{\bm{\Phi}}, \\
\mathcal{M}(T_0) 
= &\sqrt{
A\,\left(\mathcal{E}_{r}(T_0)
+ d\, \mathcal{E}_{c}(T_0)\right)
+ 4\,\frac{\log (T\, C)}{T_0}
}.
\label{EQ:M(T_0)}
\end{align}
\end{subequations}
Upon obtaining the estimated value $\hat{\mathrm{OPT}}(T_0)$, we set 
$
\Lambda=\frac{T}{\Phi_\text{min}} \left(\hat{\text{OPT}}\left(T_0\right)+\mathcal{M}\left(T_0\right)\right).
$
Consequently, the estimated value of $\Lambda$ effectively controls the deviation induced by estimation uncertainty.
The detailed procedure for determining $\Lambda$ is provided in Alg.~\ref{alg: Initial phase} in Appendix~\ref{appedix: initial phase}.

\subsubsection{\textbf{Exploration-Exploitation Phase}} In this phase, \emph{Scheduler} leverages the outputs of \emph{Predictor} and \emph{Constrainer} to make online decisions. In general, the exploration-exploitation phase consists of three steps in each round $t \in \left((A+1)\cdot T_0,\, T\right]$, as follows:

\textbf{Step~\ding{182} (Generating Action Scores)}: After observing multi-modal context $\bm{x}^t$, \emph{Scheduler} receives $\hat{r}_a^t$ and $\hat{\bm{\varphi}}_a^t$ from \emph{Predictor} and the current multiplier $\bm{\lambda}^t$ from \emph{Constrainer} for all $a\in\mathcal{A}$, and then computes the predicted per-round Lagrangian as the action score:
\begin{equation}
\label{eq:score function}
\mathcal{S}^t(a)=\hat{r}_a^t-\langle\frac{\hat{\bm{\varphi}}^{t}_a}{\bm{\Phi}} -\frac{1}{T},\bm{\lambda}^{t}\rangle,\forall a \in\mathcal{A}.
\end{equation}
\textbf{Step~\ding{183} (Selecting the Optimal Action)}: Based on the estimated action scores, \emph{Scheduler} assigns exploration-exploitation sampling probabilities to all arms. Let $a^t_{\max}$ denote the action with the highest score $\mathcal{S}^{t}(a^t_{\max})$; the probability of sampling other action is then computed as:
\begin{equation}
\label{eq: probability_action_a}
    \mathcal{P}^{t}(a)=\frac{1}{A+\rho\,\bigl(\,\mathcal{S}^{t}\left(a^t_\text{max}\right)-\mathcal{S}^{t}\left({a}\right)\,\bigl)}, \forall a\in\mathcal{A}\,\backslash\, a^t_\text{max}, 
\end{equation}
where $\rho$ is a hyper-parameter controlling the extent to which \emph{M$^2$-CMAB} favors exploitation over exploration. When $\rho$ is small, $\mathcal{P}^{t}(a)$ approaches $\frac{1}{A}$, indicating that \emph{M$^2$-CMAB} samples actions more uniformly and thus prefers exploration. The selection probability of $a^t_{\text{max}}$ is given by 
\begin{equation}
\label{eq: probability_action_b}
    \mathcal{P}^{t}(a^t_\text{max})=1-\sum_{a\in\mathcal{A}, a \neq a^t_\text{max}} \mathcal{P}^{t}(a).
\end{equation}
The final action $a^{t}\sim\mathcal{P}^t(\cdot)$ is sampled according to the probability distribution defined in Eqs.~\eqref{eq: probability_action_a} and~\eqref{eq: probability_action_b}, under the context $\bm{x}^{t}$ and the variable $\bm{\lambda}^{t}$.

\textbf{Step~\ding{184} (Updating Penalties)}:
\emph{M$^2$-CMAB} leverages Constrainer to update the new dual variable vector $\bm{\lambda}^{t+1}$ based on Eq.~\eqref{eq:lambdaupdate}, which will be used for budget allocation in the next round on the fly.
This process navigates a dynamic stream of multi-modal inference tasks, where each task is strategically scheduled to maximize inference performance while adhering to resource constraints. The details of \emph{M$^2$-CMAB} are shown in Alg.~\ref{alg:The M$^2$-CMAB Algorithm} in Appendix.~\ref{appendix: Regret Analysis Details}.
\section{Regret Analysis} \label{sec: Regret Analysis}
We next analyze the regret performance of \emph{M$^2$-CMAB}. 
Following~\cite{Zhang2025TSC}, we define $\mathrm{OPT}$ in Definition~\ref{def:OPT} in Appendix~\ref{appendix: Regret Analysis Details} as the optimal expected per-round reward achievable under the prescribed resource budget constraints.
Let $T\cdot \text{OPT}$ denote an upper bound on the cumulative optimal reward over $T$ rounds in the online setting~\cite{Agrawal2016}. 
The regret of \emph{M$^2$-CMAB} is defined as $\mathrm{Reg}(T)=T\cdot \text{OPT}-\mathbb{E}\!\left[\sum_{t=1}^T r^t_{a}\right]$, and we establish the following regret guarantee.
\begin{theorem} \label{theorem: regret}
Let $T$ denote the total number of scheduling rounds, $T_0$ the number of times each action is selected in the initial phase, and $A$ the number of actions. Define $\Phi_{\mathrm{min}}=|\bm{\Phi}^{-1}|_\infty$. Then, we have:
\begin{equation}
\label{eq:regret_bound}
\begin{aligned}
\operatorname{Reg}(T)
&\;\lesssim\;
\left(\frac{T\,\mathrm{OPT}}{\Phi_{\min}}+1\right)
\Bigl(
 A T_0\;+ \\
& \sqrt{A\,T\!\left[\operatorname{Reg}^r(T)+\operatorname{Reg}^c(T)+
\log(C\,T)\right]}
\Bigr),
\end{aligned}
\end{equation}
where $\Phi_{\mathrm{min}} > \max\left\{(A+2)\,T_0,\; T\,\mathcal{M}(T_0)\right\}$ characterizes the feasible choice of $T_0$ and ensures that all budget constraints are satisfied.
Here, $\mathcal{M}(\cdot)$ is defined in~\eqref{EQ:M(T_0)}.
Moreover, $\operatorname{Reg}^r(T)$ and $\operatorname{Reg}^c(T)$ denote the cumulative regrets incurred by the reward and cost predictors, respectively, with respect to the optimal regression functions $\hat{\mathcal{R}}^*(a,\bm{x})$ and $\hat{\bm{\mathcal{C}}}^*(a,\bm{x})$ over $T$ rounds.
\end{theorem}
\begin{proof}
The detailed proof is provided in Appendix~\ref{appendix: Regret Analysis Details}.
\end{proof}
\begin{remark}[\textit{Discussions and Open Issues}]
The regret bound in Theorem~\ref{theorem: regret} is modular with respect to the choice of the reward and cost estimators and holds for any estimator achieving sublinear estimation regret~\cite{Wu2015}. 
Under this abstraction, existing results on GP-based predictors with kernelized confidence bounds imply that, assuming unbiased cost estimation, the induced estimation regret is on the order of $\mathrm{O}\!\left(T^{3/4}(\log T)^{1/4}\right)$~\cite{Zhang2025TSC}. 
In a similar spirit, prior work on neural contextual bandits indicates that neural estimators can also achieve sublinear regret (e.g., $\tilde{\mathrm{O}}(\sqrt{T})$) under appropriate regularity and over-parameterization assumptions when coupled with confidence-based exploration mechanisms~\cite{Li2025}. 
Extending such estimation regret guarantees to more general MLLM-based predictors under the decoupled scheduling scheme considered in this work remains an open problem.
\end{remark}

\begin{figure*}[tp]
  \centering
  \includegraphics[width=\textwidth]{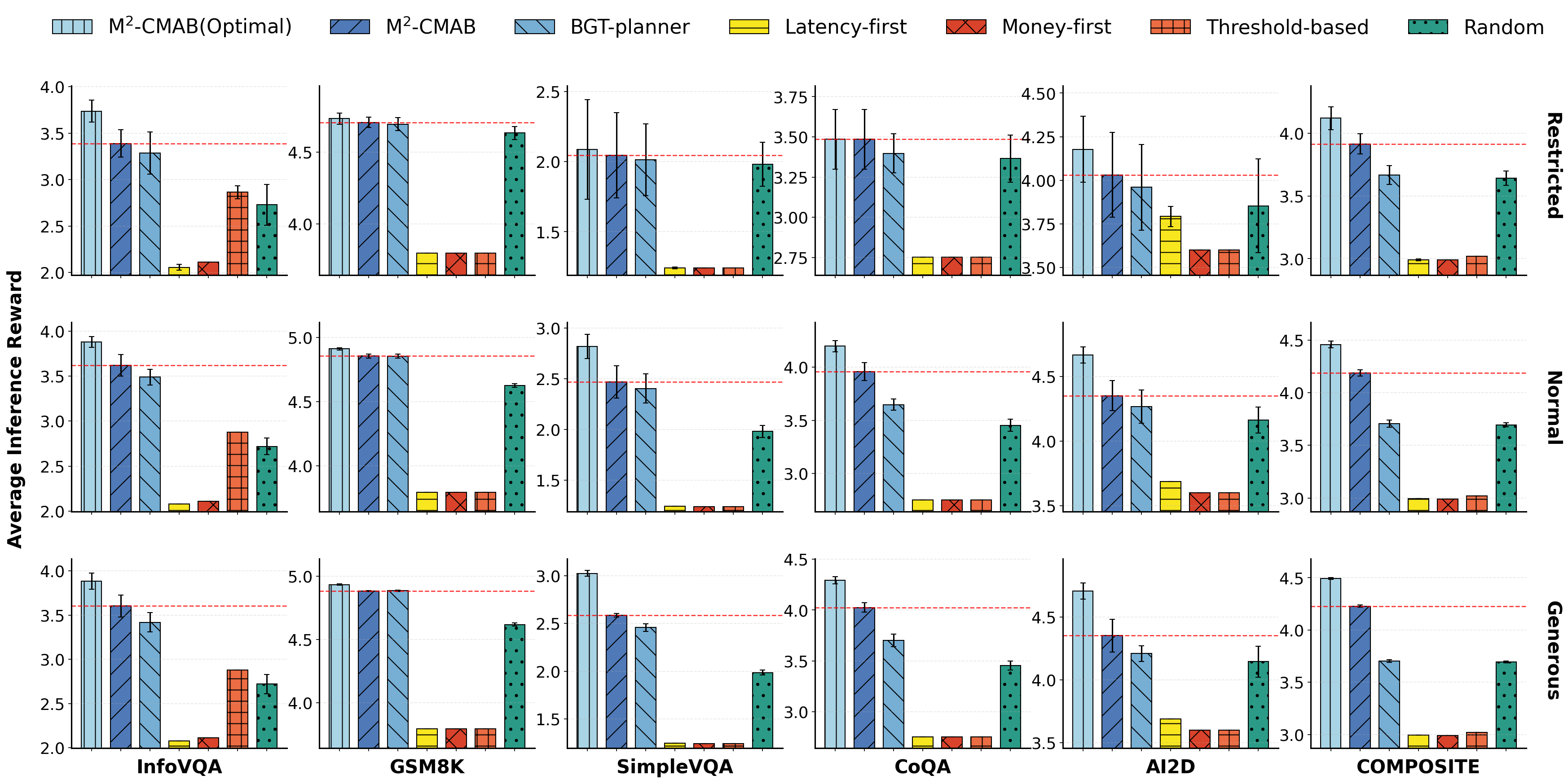}
  \caption{Average inference reward of \emph{M$^2$-CMAB} and baselines across six datasets and three budget regimes.}
  \label{fig:main-results}
  \vspace{-4mm}
\end{figure*}

\section{Performance Evaluation}
\label{sec:evaluation}

This section presents the experimental methodology for evaluating the proposed MLLM scheduling framework.
Specifically, we aim to answer three questions:
\textbf{Q1:} How robustly the framework performs under different budget regimes and task data distributions;
\textbf{Q2:} How sensitive the framework is to key hyperparameters of \textit{Scheduler};
\textbf{Q3:} How individual design adapters contribute to the final performance.

\subsection{Benchmark}
\label{sec:benchmark}
We build a trace-driven online multi-modal inference benchmark based on execution traces measured across multiple heterogeneous backends, comprising six datasets in total: five individual datasets and one composited-task dataset. Due to space constraints, some detailed descriptions of the benchmark setup are provided in Appendix~\ref{appendix: supply experiments}.

\textbf{Backends.}
The benchmark consists of five different execution backends for each incoming task, covering both local and cloud-based MLLMs.
Specifically, the local execution options include lightweight models deployed on-device, while the cloud execution options comprise large-scale, reasoning-enhanced models provided by commercial platforms, spanning diverse choices in model scale, architecture (mixture-of-experts (MoE) vs.\ non-MoE), inference mode (thinking or not), and deployment platform.

\textbf{Datasets.} Our evaluation uses six datasets, including five representative individual datasets covering diverse tasks and modality requirements, as well as one composited-task dataset:
InfoVQA~\cite{mathew2021infographicvqa}, GSM8K~\cite{cobbe2021gsm8k}, SimpleVQA~\cite{cheng2025simplevqa}, CoQA~\cite{reddy-etal-2019-coqa}, and AI2D~\cite{kembhavi2016diagram}. 
These datasets span a diverse set of tasks, including multi-turn dialogue and reading comprehension (e.g., CoQA), mathematical reasoning (e.g., GSM8K), multi-modal question answering and diagram-based visual reasoning (e.g., SimpleVQA and AI2D), as well as knowledge-intensive visual understanding (e.g., InfoVQA).
In addition to evaluating each dataset individually, we merge them to form a unified, large-scale dataset (COMPOSITE), enabling comprehensive evaluation across heterogeneous domain tasks.
In evaluation, the measured inference traces of each dataset are split into two disjoint subsets. One half is used for offline adapter preparation, with an 80/20 split for training and validation, while the other half is reserved for online evaluation, where tasks are randomly permuted and then revealed sequentially for online decision making and parameter updates~\cite{Zhang2026}.

\textbf{Inference Metadata.}
Following the protocol of~\citet{Yuan2025}, we execute each task instance on candidate execution models and record the resulting inference metadata.
System performance is evaluated along three dimensions:
\ding{182}~\textit{Inference quality.}
Response quality is quantified by a normalized reward $r \in [1,5]$, where responses from different task types are mapped to a unified five-level quality scale to ensure consistency across datasets.
\ding{183}~\textit{Latency.}
Latency $\varphi_{\text{LATENCY}}$ measures the end-to-end inference delay, including computation and network transmission when applicable. 
\ding{184}~\textit{Monetary cost.}
Monetary cost $\varphi_{\text{COST}}$ measures the economic expense of inference execution.
For local execution, it reflects energy consumption translated into cost, whereas for cloud execution it includes token-based billing and modality-specific charges.
Additional details of the inference metadata are provided in Appendix~\ref{appendix: supply experiments}.



\textbf{Baselines.} 
We evaluate our approach against several representative baselines: 
\ding{182}~\textit{Random}, which selects an execution backend uniformly at random in each round;
\ding{183}~\textit{Latency-first}, which greedily selects the execution model with the lowest predicted latency according to our trained latency predictor;
\ding{184}~\textit{Money-first}, which greedily selects the execution model with the lowest predicted monetary cost using our trained cost predictor; 
\ding{185}~\textit{BGT-planner}~\cite{Zhang2025TSC}, a state-of-the-art CMAB-based budget allocation framework; 
\ding{186}~\textit{Threshold-based}~\cite{Li2023,Zhang2026}, an online policy that selects actions based on a predicted utility-to-cost ratio, where multiple predicted cost dimensions are aggregated by averaging to adapt the original single-cost formulation
\ding{187}~\textit{Optimal}, an oracle-aided upper bound that makes decisions using perfect per-round reward and cost observations, representing the best achievable performance of \emph{M$^2$-CMAB} under ideal estimation.
For \ding{185} and \ding{186}, we adhere to the experimental settings and hyperparameter configurations reported in the original works.

\subsection{Evaluation Configuration}
\label{sec:system}

\textbf{Backbone Model.}
The scheduling policy is driven by a locally deployed \textit{Qwen3-VL-2B-Instruct} model with a frozen backbone, which produces task-level multimodal representations for scheduling. Only lightweight reward and cost adapters are trained to predict the expected reward, latency, and monetary cost given the task context and candidate actions. 
When a local execution decision is selected, the cached intermediate states are reused and the original output head is invoked to directly generate the final response. 

\textbf{Budget regimes.}
We consider three budget regimes for both latency and monetary cost, derived from the aggregated cost of the five candidate actions. 
Specifically, \emph{Restricted} sets the budget to the minimum aggregated cost value across all actions, \emph{Normal} to the second-smallest aggregated cost, and \emph{Generous} to the median aggregated cost.


\textbf{Hyperparameters.}
The key hyperparameters used in the scheduling and learning framework are summarized in Appendix~\ref{app:impl}.
In addition, we study the effect of the scheduling parameter $T_0$ through a sensitivity analysis on the fraction of rounds assigned to the initial phase.

\begin{figure*}[tp]
  \centering
      \includegraphics[width=\textwidth]{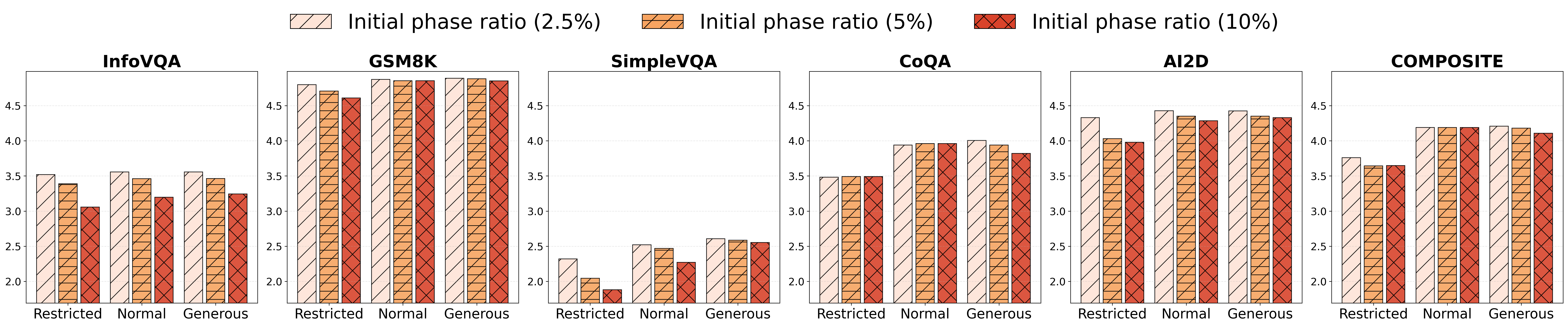}
    \caption{Sensitivity of \emph{M$^2$-CMAB} to the initial phase ratio, i.e., $\frac{(A+1)T_0}{T}$, evaluated by average inference reward (y-axis). Bars indicate the percentage of total rounds allocated to the initial phase, with three groups per subfigure corresponding to three budget regimes.}
  \label{fig:robustness}  
    \includegraphics[width=\textwidth]{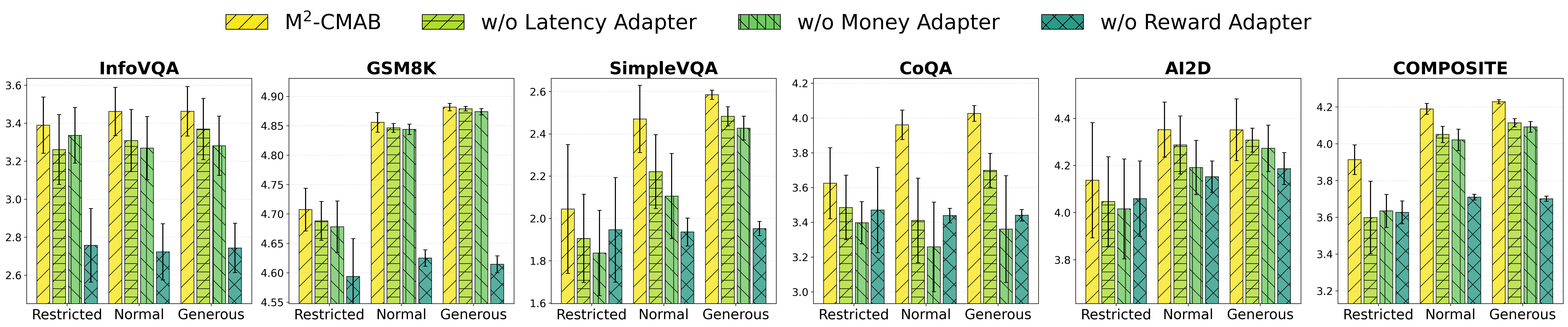}
      \caption{Ablation study of \emph{M$^2$-CMAB} across six datasets and three budget levels, illustrating the contribution of latency, money and reward predictor to the average inference reward (y-axis).}
  \label{fig:ablation}  
  \vspace{-5mm}
\end{figure*}

\subsection{Experimental Results and Analysis}

\textbf{Overall Performance Comparison (Answer for Q1).} Since each algorithm terminates upon exceeding either the latency or monetary budget, resulting in varying numbers of executed rounds, we adopt the \textbf{average inference reward} across executed rounds as the primary evaluation metric to ensure fair comparison.
Each experiment is repeated five times with different random seeds to account for variability and ensure statistical reliability.
As shown in Fig.~\ref{fig:main-results}, \emph{M$^2$-CMAB} consistently outperforms all realistic baselines across datasets and achieves performance closest to the oracle-aided upper bound.
Notably, our method outperforms the second-best baseline by 6.79\%, 13.08\%, and 14.18\% under the Restricted, Normal, and Generous budget regimes, respectively, on the \textsc{COMPOSITE} trace. 
These results underscore the effectiveness of the proposed representation and scheduling design in \emph{M$^2$-CMAB} for highly heterogeneous and complex MLLM inference workloads.

Moreover, as the budget increases, \emph{M$^2$-CMAB} achieves greater average improvements over the baselines, indicating that the algorithm can more effectively leverage additional scheduling flexibility to optimize performance. 
In some settings, the performance gap between \emph{M$^2$-CMAB} and the second-best baseline becomes smaller, particularly under highly constrained budgets (e.g., CoQA and GSM8K under the Restricted setting) or when task rewards are relatively uniform across actions (e.g., GSM8K). The distributions of the datasets can be found in Appendix~\ref{app:impl}.
In these cases, the inherently limited room for task inference scheduling naturally reduces the potential performance margin across different decision policies. 
Notably, even under such challenging regimes, \emph{M$^2$-CMAB} remains very close to the Optimal oracle, with a gap below 1.2\%, thereby demonstrating the robustness of the proposed approach in extreme and resource-constrained deployment scenarios.


\textbf{Hyperparameter Sensitivity (Answer for Q2). }
We also conduct experiments to examine the effect of varying the initial phase ratio in Fig.~\ref{fig:robustness}.
Specifically, we consider initial phase ratios of 2.5\%, 5\%, and 10\% of the total rounds. 
Overall, \emph{M$^2$-CMAB} maintains stable and robust performance across a reasonable range of this hyperparameter, with a slight decrease in average inference reward as the initial phase ratio increases, particularly under the Restricted budget condition.
This behavior can be attributed to the role of the initial phase: while allocating more rounds to this phase improves the accuracy of estimating the effective range of the constraint-control variables $\bm{\lambda}$, it simultaneously reduces the number of rounds available for subsequent exploration and exploitation.
This trade-off suggests that the length of the initial phase should be chosen in accordance with the underlying cost distribution and budget regime, rather than being set excessively large.

\textbf{Ablation Study (Answer for Q3). } 
To evaluate the contribution of each component in \emph{M$^2$-CMAB}, we perform ablation experiments on six datasets under three budget regimes. 
Specifically, we replace each key adapter-based predictor with a random predictor while keeping the other components unchanged.  
As shown in Fig.~\ref{fig:ablation}, removing any component results in a performance drop relative to the complete version of our framework. 
In particular, removing the reward adapter from \emph{M$^2$-CMAB} leads to a more substantial  performance degradation compared to the other adapters.
This is likely because the Money and Latency adapters primarily function as predictors for budget constraints, and inaccuracies in one can be partially compensated by the other. However, accurate predictions from the Reward adapter are essential for effective scheduling decisions in \emph{M$^2$-CMAB}.


\section{Conclusion}\label{sec:conclusion}
In this work, we address two fundamental challenges in MLLM inference scheduling: efficient task representation under heterogeneous modalities and effective resource utilization in online, budget-constrained settings. To this end, we propose \emph{M$^2$-CMAB}, an online scheduling framework for resource-constrained MLLM inference task.
At its core, \emph{M$^2$-CMAB} employs a CLS-attentive multi-adapter \emph{Predictor} for lightweight reward and cost estimation. A budget-aware \emph{Constrainer} adjusts the relative importance of resource consumptions across rounds, enabling adaptive constraint handling. Building on these components, a two-phase \emph{Scheduler} balances exploration and exploitation to ensure stable performance under diverse resource budgets.
Future work will focus on establishing online regret guarantees for MLLM-based estimators used in the \emph{Predictor}, and on exploring more lightweight and finer-grained representations of tasks and execution backends to further improve robustness and efficiency in large-scale MLLM deployments.

%% file: appendix.tex
\section{Table of Main Notations}\label{appendix: Major Notations}

Table~\ref{Table:Major Notations} provides the definitions of the main notations and symbols used throughout the paper to facilitate understanding and reference.
\begin{table}[!h]
\centering
\caption{Main notations and definitions.}
\label{Table:Major Notations}
\renewcommand{\arraystretch}{1.2}
\rowcolors{2}{white}{gray!20}
\begin{tabular}{p{3cm}<{\centering} m{\linewidth-3.9cm}}
\toprule
\multicolumn{1}{c}{\textbf{Notation}} &
\multicolumn{1}{c}{\textbf{Definition}} \\
\midrule
$t \ /\ T$ 
& Index of the online scheduling round and the total number of rounds. \\

$\Omega$ 
& Set of supported input modalities (e.g., text, vision, audio). \\

$\bm{x}^t = \{ \bm{x}^t_\omega \}_{\omega\in\Omega}$ 
& Multi-modal context of the task arriving at round $t$, consisting of modality-specific token embeddings. \\

$a^t \in \mathcal{A}$ 
& Action selected at round $t$, corresponding to an MLLM inference configuration. \\

$A$ 
& Cardinality of the action space. \\

$r_a^t$ 
& Observed reward of executing action $a$ under context $\bm{x}^t$ at round $t$. \\

$\bm{\varphi}_a^t \in \mathbb{R}^C_+$ 
& Observed resource consumption (cost) vector of action $a$ at round $t$ over $C$ constraint dimensions. \\

$\bm{\Phi} \in \mathbb{R}^C_+$ 
& Long-term budget vector across all resource dimensions. \\

$\mathcal{R}(a,\bm{x})$ 
& Unknown expected reward function induced by action $a$ and context $\bm{x}$. \\

$\mathcal{\bm{C}}(a,\bm{x})$ 
& Unknown expected cost function mapping an action-context pair to a $C$-dimensional resource consumption vector. \\

$\hat r_a^t = \hat{\mathcal{R}}(a,\bm{x}^t;\Theta_r^t)$ 
& Predicted reward of action $a$ at round $t$ given by the adapter-based predictor. \\

$\hat{\bm{\varphi}}_a^t = \hat{\bm{\mathcal{C}}}(a,\bm{x}^t;\bm{\Theta}_\varphi^t)$ 
& A collection of $C$ unknown cost functions, where each 
$\mathcal{C}(a,\bm{x};\bm{\Theta}_\varphi^{t,(c)})$ maps an action-context pair to the expected 
consumption of the $c$-th resource dimension. \\

$f_{\bm{\vartheta}}(\cdot)$ 
& Frozen MLLM backbone with parameter $\bm{\vartheta}$ used to extract task-level multi-modal representations. \\

$\bm{z}_x^t$ 
& Task-level hidden representation obtained from CLS-attentive pooling over the frozen MLLM backbone at round $t$. \\

$\bm{z}_a$ 
& Embedding representation of action $a$. \\

$\bm{\Theta}_r^t$
& Parameter set of the reward predictor at round $t$, consisting of the
frozen MLLM backbone parameters $\bm{\vartheta}$ and the trainable reward
adapter parameters $\bm{\theta}_r^t$. \\

$\bm{\Theta}_{\varphi}^{t}$
& Parameter set of the cost predictors at round $t$, given by
$\{\bm{\Theta}_{\varphi}^{t,(c)}\}_{c=1}^C$, where
$\bm{\Theta}_{\varphi}^{t,(c)}=\{\bm{\vartheta},\bm{\theta}_{\varphi}^{t,(c)}\}$
denotes the parameters of the $c$-th cost adapter. \\

$\bm{\lambda}^t \in \mathbb{R}^C_+$ 
& Vector of dual variables (Lagrange multipliers) associated with long-term  constraints at round $t$. \\

$\mathcal{V}$ 
& Feasible domain of the dual variables with $l_1$-radius $\Lambda$. \\

$\Lambda$ 
& Radius of the dual feasible set, determined in the initial phase to control budget violation. \\

$T_0$ 
& Number of exploration rounds per action in the initial phase of M$^2$-CMAB. \\

$\mathcal{S}_t(a)$ 
& Predicted per-round Lagrangian score of action $a$ at round $t$, combining reward and penalized cost. \\

$\mathcal{P}_t(a)$ 
& Sampling probability of selecting action $a$ at round $t$ in the exploration–exploitation phase \\

$\mathrm{Reg}(T)$ 
& Regret of M$^2$-CMAB over $T$ rounds, defined as $T\cdot \mathrm{OPT} - \mathbb{E}\!\left[\sum_{t=1}^T r_{a^t}^t\right]$. \\

$\mathrm{Reg}^r(T),\ \mathrm{Reg}^c(T)$ 
& Cumulative estimation regrets of the reward predictor and cost predictors over $T$ rounds. \\

$\Phi_{\min} = \|\bm{\Phi}^{-1}\|_{\infty}$ 
& Minimum normalized budget parameter used in regret analysis and feasibility conditions. \\

\bottomrule
\end{tabular}
\end{table}

\section{Details of Adapter-Based Predictors}
\label{appendix:adapter_details}

\paragraph{Context-Aware Representation Extraction.}
Given an action $a \in \mathcal{A}$, we consider the multi-modal context input
$\bm{x}^t = \{ \bm{x}^t_\omega \}_{\omega \in \Omega}$ together with the action
embedding $\bm{z}_{a}$ to predict the corresponding reward and cost.
To capture the intrinsic semantic complexity of the task, we leverage a
pre-trained multi-modal large language model (MLLM) backbone
$f_{\bm{\vartheta}}(\cdot)$ with frozen parameters $\bm{\vartheta}$, and
extract a task-level representation $\bm{z}_{x}^{t}$ from the multi-modal
context.

Specifically, an explicit \texttt{[CLS]} token is prepended to the multi-modal
token sequence $\bm{x}^t$ as a global semantic anchor. From the selected output
layer of the frozen backbone, we extract the multi-head self-attention tensor
$\bm{\mathrm{I}}^{t} \in \mathbb{R}^{H \times L \times L}$ and the hidden-state
matrix $\bm{\mathrm{H}}^{t} \in \mathbb{R}^{L \times d_{\mathrm{hid}}}$, where
$H$ denotes the number of attention heads, $L$ the token sequence length, and
$d_{\mathrm{hid}}$ the hidden-state dimensionality.

For each attention head $h \in \{1,\ldots,H\}$, we extract the attention weights
from the \texttt{[CLS]} token to all tokens as
\begin{equation}
\bm{\alpha}_h^t
= \bm{\mathrm{I}}_{h}^{t}[\mathrm{CLS},:] \in \mathbb{R}^{L}.
\end{equation}
Using these weights, we perform attention-based pooling over the hidden states
to obtain a head-wise semantic summary. To avoid introducing additional
trainable parameters, the pooled representations across all heads are aggregated
by simple averaging, yielding a compact, parameter-free context embedding
\begin{equation}
\bm{z}_{\bm{x}}^t
= \frac{1}{H} \sum_{h=1}^{H} \frac{1}{L}
\sum_{l=1}^{L} \alpha_{h,l}^t \, \bm{h}_l^t,
\end{equation}
where $\bm{h}_l^t$ denotes the $l$-th hidden-state vector from
$\bm{\mathrm{H}}^{t}$. The resulting representation $\bm{z}_{\bm{x}}^t$ is used
as the contextual feature fed into the bandit algorithm.

\paragraph{Backbone Model.}
The task backbone $f_{\bm{\vartheta}}(\cdot)$ is instantiated using
\emph{Qwen3-VL-2B-Instruct}~\footnote{\url{https://huggingface.co/Qwen/Qwen3-VL-2B-Instruct}},
a lightweight yet powerful multi-modal large language model capable of jointly
processing text and visual inputs. Its parameters remain frozen throughout
training, serving as a shared semantic encoder across all tasks and constraints.

\paragraph{Action Embedding.}
The embedding $\bm{z}_{a}$ of action $a$ is obtained using a dedicated language
embedding model, ensuring semantic alignment between the action description and
the contextual representation. In our implementation, we employ
\emph{Qwen3-Embedding-0.6B}~\footnote{\url{https://huggingface.co/Qwen/Qwen3-Embedding-0.6B}}
to encode each candidate action into a fixed-dimensional vector, which is then
concatenated with $\bm{z}_{\bm{x}}^t$ for downstream prediction.

\paragraph{Adapter-Based Reward and Cost Prediction.}
Conditioned on both the context embedding $\bm{z}_{\bm{x}}^t$ and the action
embedding $\bm{z}_{a}$, prediction is performed via lightweight task-specific
adapters. Concretely, we concatenate the two embeddings and feed them into a set
of adapters $\mathcal{F}_{\bm{\theta}}(\bm{z}_{a} || \bm{z}_{\bm{x}}^t)$,
parameterized by $\bm{\theta}_r$ for reward prediction and
$\{\bm{\theta}_{\varphi}^{(c)}\}_{c=1}^{C}$ for per-dimension cost prediction.

Each adapter is instantiated as a lightweight multi-layer perceptron (MLP),
whose detailed architecture is provided in the experimental section.
Importantly, all adapters are trained while keeping the backbone parameters
$\bm{\vartheta}$ frozen, ensuring stable representation learning and
efficient online updates.

Formally, the reward predictor is defined as
\begin{equation}
\hat{r}_a^t = \hat{\mathcal{R}}(a, \bm{x}^t ; \bm{\Theta}_r^t),
\qquad
\bm{\Theta}_r^t = \{\bm{\vartheta}, \bm{\theta}_r^t\},
\end{equation}
and for each budget-constrained resource dimension $c \in \{1,\ldots,C\}$, the
corresponding cost predictor is given by
\begin{equation}
\hat{\varphi}_a^{t,(c)} =
\hat{\mathcal{C}}(a, \bm{x}^t ; \bm{\Theta}_{\varphi}^{t,(c)}),
\qquad
\bm{\Theta}_{\varphi}^{t,(c)} = \{\bm{\vartheta}, \bm{\theta}_{\varphi}^{t,(c)}\}.
\end{equation}
The resulting vector $\hat{\bm{\varphi}}_a^t$ characterizes the estimated
per-round resource consumption associated with action $a$.

\paragraph{Training Objective.}
Both reward and cost adapters are trained using a unified regularized
least-squares objective. Let $\mathcal{F}_{\bm{\theta}}$ denote either a reward
or a cost predictor, and let $y^\tau$ be the corresponding supervision signal,
\emph{i.e.}, reward $r^\tau$ or cost $\varphi^{\tau,(c)}$. The adapter parameters
are updated by solving
\begin{equation}
\min_{\bm{\theta}}
\; \mathcal{J}(\bm{\theta})
= \frac{1}{2|\Upsilon^t|}
\sum_{\tau < t} \big(\mathcal{F}_{\bm{\theta}} - y^\tau\big)^2
+ \frac{\lambda}{2} \| \bm{\theta} - \bm{\theta}^0 \|_2^2,
\end{equation}
where
$\Upsilon^t = \{(r^\tau, \bm{\varphi}^\tau, a^\tau, \bm{x}^\tau)
\mid \tau < t, \tau \in \mathbb{N}_+\}$ denotes the historical observation set,
$\bm{\theta}^0$ is the initialization of adapter parameters, and $\lambda > 0$
controls the regularization strength.

\begin{algorithm}[t]
\caption{Initial Phase of M$^2$-CMAB for Estimating Dual Radius $\Lambda$}
\label{alg: Initial phase}
\textbf{Input:}
Total rounds $T$; exploration length $T_0$;
action set $\mathcal{A}$ with $|\mathcal{A}|=A$. \\
\textbf{Output:}
Estimated dual radius $\Lambda$;
historical dataset $\Upsilon^{(A+1)T_0}$.

\begin{algorithmic}[1]
\STATE Initialize historical dataset $\Upsilon^0 \leftarrow \emptyset$.

\textbf{// Stage I //}
\STATE Initialize action counters $n_a \leftarrow T_0$, for all $a \in \mathcal{A}$.
\STATE Initialize historical dataset $\Upsilon^0 \leftarrow \emptyset$.
\FOR{$t = 1$ \textbf{to} $A \cdot T_0$}
    \STATE Observe multi-modal context
    $\bm{x}^t = \{ \bm{x}^t_\omega \}_{\omega \in \Omega}$.
    \STATE Randomly sample an action $a^t$ from
    $\{ a \in \mathcal{A} : n_a > 0 \}$.
    \STATE Execute action $a^t$ and observe reward $r_{a}^t$ and cost vector $\bm{\varphi}^t$ for all budgets.
    \STATE Update counter $n_{a^t} \leftarrow n_{a^t} - 1$.
    \STATE Update dataset
    $\Upsilon^t \leftarrow \Upsilon^{t-1}
    \cup \{ (r^t, \bm{\varphi}^t, a^t, \bm{x}^t) \}$.
\ENDFOR

\textbf{// Stage II //}

\FOR{$t = AT_0 + 1$ \textbf{to} $(A+1)T_0$}
    \STATE Select an exploratory action $a^t \in \mathcal{A}$.
    \STATE Observe multi-modal context $\bm{x}^t$.
    \STATE Observe reward $r_{a}^t$ and cost vector $\bm{\varphi}^t$ for all budgets.
    \STATE Update dataset
    $\Upsilon^t \leftarrow \Upsilon^{t-1}
    \cup \{ (r^t, \bm{\varphi}^t, a^t, \bm{x}^t) \}$.
    \STATE Compute predicted rewards $\hat r_a^t = \hat{\mathcal{R}}\left(a,\bm{x}^t \, ;\, \bm{\Theta}^{t}_r \right)$
for all $a \in \mathcal{A}\, / \,a^t$.
    \FOR{$c = 0$ \textbf{to} $C-1$}
    \STATE Compute predicted cost  $\hat{\varphi}^{t,(c)}_{a} 
    = \hat{\mathcal{C}}\!\left(a,\bm{x}^t\, ;\, \bm{\Theta}^{t,{(c)}}_\varphi \right)$
for all $a \in \mathcal{A}\, / \,a^t$.
\ENDFOR
\ENDFOR

\STATE Estimate $\hat{\mathrm{OPT}}(T_0)$ by solving the linear program
defined in Eq.~\eqref{eq:linear_programming}.

\STATE Set the dual radius $\Lambda=\frac{T}{\Phi_\text{min}} \left(\hat{\text{OPT}}\left(T_0\right)+\mathcal{M}\left(T_0\right)\right)$, where $\Phi_\text{min} = \|\bm{\Phi}^{-1}\|_\infty$ and $\mathcal{M}\left(T_0\right)$ is defined in Eq.~\eqref{EQ:M(T_0)}.

\end{algorithmic}
\end{algorithm}

\section{Initial Phase of M$^2$-CMAB}
\label{appedix: initial phase}
This appendix describes the initial phase of \emph{M$^2$-CMAB}, which is designed
to determine the radius $\Lambda$ of the dual feasible set
$\mathcal{V} = \{ \bm{\lambda} \in \mathbb{R}^C_+ : \| \bm{\lambda} \|_1 \le \Lambda \}$.
The purpose of this phase is to ensure sufficient exploration of the action
space and to obtain a reliable estimate of the optimal value required for
subsequent dual-based exploration--exploitation.

The initial phase consists of two sequential stages.
First, each action $a \in \mathcal{A}$ is executed for $T_0$ rounds to collect
context--reward observations, which are used to fit a regression model for
estimating the unknown expected reward function $\mathcal{R}(a,\bm{x})$ and reward functions $\bm{\mathcal{C}}(a,\bm{x})$.
Second, an additional exploration stage is conducted to refine the reward
estimator and to compute an empirical estimate of the optimal value
$\hat{\mathrm{OPT}}(T_0)$, based on which the dual radius $\Lambda$ is
determined.

Throughout this phase, only observed rewards are used for estimation, and no
dual variables are updated. The collected historical data serve as a warm-start
for both the predictor initialization and the dual-constrained optimization in
the main phase.

\begin{algorithm}[t]
\caption{The M$^2$-CMAB Algorithm}
\label{alg:The M$^2$-CMAB Algorithm}
\begin{algorithmic}[1]
\REQUIRE Total rounds $T$; Exploration rounds $T_0$.

\STATE Determine the dual radius $\Lambda$ via the initial phase
(Alg.~\ref{alg: Initial phase}) and obtain the historical set
$\Upsilon^{(A+1)\cdot T_0}$;
\STATE Initialize the dual variable $\bm{\lambda}^1 \in \mathcal{V}$;
\FOR{$t = (A+1)\cdot T_0 + 1, \ldots, T$}
    \STATE Observe the context of the $t$-th task $\bm{x}^t$;

    \STATE \textbf{// Step~\textcircled{1}: Generating Action Scores //}
    \STATE Use the reward and cost predictors to estimate the expected
    reward $\hat r_a^t$ and cost vector $\hat{\bm{\varphi}}_a^t$
    for all $a\in\mathcal{A}$ based on $\bm{x}^t$ and action embeddings;

    \STATE Compute the per-round Lagrangian score
    $\mathcal{S}^t(a)$ for each action $a\in\mathcal{A}$
    using the current dual variable $\bm{\lambda}^t$;

    \STATE \textbf{// Step~\textcircled{2}: Selecting the Optimal Action //}
    \STATE Construct the action sampling distribution
    $\mathcal{P}^t(\cdot)$ according to the scores
    $\{\mathcal{S}^t(a)\}_{a\in\mathcal{A}}$;
    \STATE Sample the scheduling action $a^t \sim \mathcal{P}^t$;

    \STATE \textbf{// Step~\textcircled{3}: Updating Penalties //}
    \STATE Observe the realized reward $r^t$ and cost vector
    $\bm{\varphi}^t$, and update the historical set
    $\Upsilon^t \leftarrow \Upsilon^{t-1} \cup
    (r^t, \bm{\varphi}^t, a^t, \bm{x}^t)$;
    \STATE Update the reward and cost adapter parameters using
$\Upsilon^t$ according to the prescribed update rules;
    \STATE Update the dual variable $\bm{\lambda}^{t+1}$ via
    OMD method;
\ENDFOR
\end{algorithmic}
\end{algorithm}

\section{Regret Analysis Details}\label{appendix: Regret Analysis Details}
This appendix establishes the regret guarantees of M$^2$-CMAB shown in Alg.~\ref{alg:The M$^2$-CMAB Algorithm} in Theorem~\ref{theorem: regret}.
We begin by specifying the stochastic data-generating process and the class of
admissible reward and cost functions.
We then define the optimal static benchmark policy and its associated value,
which serves as the comparator in the regret definition.
Under these preliminaries, the regret analysis proceeds by bounding the error
introduced by function estimation and dual-based decision making.

\begin{proof}
We first formalize the stochastic environment through a realizability assumption
on the underlying reward and cost functions.

\begin{assumption}[Realizability of Reward and Cost Functions]
\label{assump:realizability_m2cmab}
There exist some optimal regression functions $\hat{\mathcal{R}}^*$ and $\hat{\bm{\mathcal{C}}^*}$ characterizing the expected reward and cost distributions, respectively, such that for
any round and any action $a \in \mathcal{A}$,
\[
\mathbb{E}\!\left[r_a \mid \bm{x} \right] = \hat{\mathcal{R}}^*(a,\bm{x}),
\qquad
\mathbb{E}\!\left[\bm{\varphi}_a \mid \bm{x} \right]
= \hat{\bm{\mathcal{C}}^*}(a,\bm{x}),
\]
where $\bm{x}=\{\bm{x}_\omega\}_{\omega\in\Omega}$ is the observed multi-modal
context.
\end{assumption}

Under Assumption~\ref{assump:realizability_m2cmab}, the stochastic reward and cost processes admit
well-defined expected functions induced by the action--context pair.
This allows us to characterize the performance of an online policy
through comparison with an optimal static benchmark that has full
knowledge of these expectations.

\begin{definition}[Optimal Static Benchmark Value]
\label{def:OPT}
Based on the realizable expected reward and cost functions in
Assumption~\ref{assump:realizability_m2cmab}, let $\pi^*:\mathcal{X}\rightarrow P_{\mathcal{A}}$ be a static randomized policy that maps each context $\bm{x}\in\mathcal{X}$ to a probability distribution over the action set $\mathcal{A}$ and $\mathbbm{1}$ represent the all-ones vector. 
Then, $\mathrm{OPT}$ is defined as
\begin{equation} 
\begin{aligned}  
 &\text{OPT} \triangleq \max_{\pi^*}\ \mathbb{E}_{\bm{x} \sim \mathcal{X}}\left[\sum_{a \in\mathcal{A}} \pi^*( \bm{x})\, \hat{\mathcal{R}}^*(a, \bm{x})\right], \\ 
 &\text { s.t. }\mathbb{E}_{\bm{x} \sim \mathcal{X}}\left[\sum_{a \in\mathcal{A}} \pi^*( \bm{x}) \, \frac{\hat{\bm{\mathcal{C}}^*}(a, \bm{x})}{\bm{\Phi}}\right] \leq \frac{\mathbbm{1}}{T}.
\end{aligned} 
\label{eq:static_policy} 
\end{equation}
\end{definition}


The benchmark value $\mathrm{OPT}$ characterizes the maximum achievable
expected reward under the long-term budget constraints in hindsight.
We now use this benchmark to formally define the regret of the
M$^2$-CMAB algorithm.

\begin{definition}[Cumulative Regret of the Online Algorithm]
\label{def:regret}

Given the optimal static benchmark value $\mathrm{OPT}$ defined in
Definition~\ref{def:OPT}, the regret of M$^2$-CMAB over $T$ rounds is defined as
\begin{equation}
\mathrm{Reg}(T)
\;\triangleq\;
T \cdot \mathrm{OPT}
-
\mathbb{E}\!\left[
\sum_{t=1}^T r_{a}^{t}
\right],
\end{equation}
where $T \cdot \mathrm{OPT}$ upper bounds the optimal cumulative reward achievable
in online scenarios~\cite{Agrawal2016}.
\end{definition}

To analyze $\mathrm{Reg}(T)$, it remains to quantify the error introduced
by learning-based estimation of the reward and cost functions.
We next formalize this error through squared-loss estimation regret.

\begin{definition}[Squared-Loss Estimation Regret]
\label{def:squared_loss_regret}
The squared-loss estimation regret of the reward predictor
$\hat{\mathcal{R}}(\cdot,\cdot)$ over $T$ rounds is defined as
\begin{equation}
\mathrm{Reg}^r(T)
\;\triangleq\;
\sum_{t=1}^T
\Big(
\hat{\mathcal{R}}(a^t,\bm{x}^t;\bm{\Theta}_r^t)
-
\hat{\mathcal{R}}^*(a^t,\bm{x}^t)
\Big)^2.
\end{equation}
Similarly, the squared-loss estimation regret of the cost predictors
$\hat{\bm{\mathcal{C}}}(\cdot,\cdot)$ is defined as
\begin{equation}
\mathrm{Reg}^c(T)
\;\triangleq\;
\sum_{t=1}^T
\Big\|
\hat{\bm{\mathcal{C}}}(a^t,\bm{x}^t;\bm{\Theta}_{\varphi}^t)
-
\hat{\bm{\mathcal{C}}^*}(a^t,\bm{x}^t)
\Big\|_2^2.
\end{equation}

\end{definition}

With the regret benchmark and estimation error measures in place,
we are ready to analyze the regret incurred during the
exploration--exploitation (main) phase of M$^2$-CMAB.
The following lemma characterizes the regret contribution of this phase
when the dual variables are updated within a bounded feasible set.

\begin{lemma}[Regret Bound of SquareCBwK, Theorem~A.1~\cite{DBLP:conf/aistats/HanZWXZ23}]
\label{lem:oracle_regret_m2cmab}
Consider the exploration-exploitation execution of M$^2$-CMAB following SquareCBwK algorithm~\cite{DBLP:conf/aistats/HanZWXZ23}, where the dual variables are updated over the feasible set
\[
\mathcal{V}
=
\big\{
\bm{\lambda}\in\mathbb{R}^C_+ :
\|\bm{\lambda}\|_1 \le \Lambda
\big\},
\]
with a fixed radius $\Lambda$ determined in the initial phase.
Then the regret incurred during the exploration-exploitation admits the bound
\begin{equation}
\mathrm{Reg}(T)
\;\lesssim\;
(\Lambda+1)
\sqrt{
A T
\Big(
\mathrm{Reg}^r(T)
+ \mathrm{Reg}^c(T)
+ \log(C\, T)
\Big)
}.
\end{equation}
\end{lemma}

Lemma~\ref{lem:oracle_regret_m2cmab} characterizes the regret incurred during
the exploration--exploitation phase when the dual variables are updated within
a fixed feasible radius $\Lambda$.
To make this result applicable, it remains to show that the dual radius
$\Lambda$, estimated in the initial phase, is well controlled with high
probability.
To this end, we establish a concentration guarantee for the estimation of the
$\ell_1$-radius $\Lambda$.

\begin{lemma}[Accuracy of the Estimated Dual Radius, Lemma~4.1~\cite{DBLP:conf/aistats/HanZWXZ23}]   
\label{lemma: OPT_Z}
Denoting the optimal value of \eqref{eq:linear_programming} by $\hat{\text{OPT}}\left(T_0\right)$, and setting $\Lambda=\frac{T}{\Phi_{\mathrm{min}}}\left(\hat{\text{OPT}}\left(T_0\right)+M\left(T_0\right)\right)$, where $\Phi_{\mathrm{min}} = \|\bm{\Phi}_{\mathrm{min}}^{-1}\|_\infty$, we have with probability at least $1-O\left(1 / T^2\right)$,
\begin{equation}
    \frac{T \cdot \text{OPT}}{\Phi_{\mathrm{min}}} \leq \Lambda \leq\left(\frac{6 T \cdot \mathcal{M}\left(T_0\right)}{\Phi_{\mathrm{min}}}+1\right)\left(\frac{T \cdot \text{OPT}}{\Phi_{\mathrm{min}}}+1\right),
\end{equation}  
\end{lemma}

Lemma~\ref{lemma: OPT_Z} establishes a high-probability bound on the
$\ell_1$-radius $\Lambda$ estimated in the initial phase.
In particular, when the normalized budget parameter satisfies
$\Phi_{\min}=\mathrm{O}\!\left(T\cdot \mathcal{M}(T_0)\right)$,
the upper bound in Lemma~\ref{lemma: OPT_Z} implies that the dual radius
can be simplified as
\begin{equation}
\Lambda \;\lesssim\; \frac{T\cdot \mathrm{OPT}}{\Phi_{\min}} + 1,
\end{equation}
where the notation ``$\lesssim$'' denotes an inequality up to a universal
constant factor.
With this estimation accuracy of $\Lambda$, the online mirror descent (OMD)
algorithm used to update the dual variable $\bm{\lambda}$ achieves an
$\mathrm{O}(\sqrt{T})$ regret guarantee, as shown in Lemma~2.2
of~\cite{DBLP:conf/aistats/HanZWXZ23}.

Substituting the above estimate of $\Lambda$ into the regret bound of the
exploration--exploitation phase given in
Lemma~\ref{lem:oracle_regret_m2cmab}, we obtain the following
high-probability regret guarantee for the exploration--exploitation phase
of M$^2$-CMAB:
\begin{equation}
\operatorname{Reg}^{\mathrm{ee}}(T)
\;\lesssim\;
\left(
\frac{T\,\mathrm{OPT}}{\Phi_{\min}} + 1
\right)
\sqrt{
A\,T\!\left(
\operatorname{Reg}^r(T)
+
\operatorname{Reg}^c(T)
+
\log(C\,T)
\right)
}.
\end{equation}

\paragraph{Regret decomposition across phases.}
In the initial phase (Alg.~\ref{alg: Initial phase}), the algorithm runs for
$A\, T_0$ rounds to estimate a feasible dual radius.
During this phase, the expected per-round regret is upper bounded by the
optimal dual radius $\Lambda$.
By Lemma~\ref{lemma: OPT_Z}, with high probability
$\Lambda \lesssim \frac{T\cdot \mathrm{OPT}}{\Phi_{\min}}+1$,
which implies that the total regret incurred in the initial phase is at most
\begin{equation}
   \operatorname{Reg}^{\mathrm{ini}}(T)
\;\lesssim\; \left(\frac{T\cdot \mathrm{OPT}}{\Phi_{\min}}+1\right) A T_0 . 
\end{equation}

Combining the regret contributions from the initial phase and the
exploration--exploitation phase, we conclude that the total regret of
M$^2$-CMAB satisfies
\begin{equation}
\begin{split}
   \operatorname{Reg}(T) \,=\, \operatorname{Reg}^{\mathrm{ini}}(T) + \operatorname{Reg}^{\mathrm{ee}}(T) 
\;\lesssim\;
\left(\frac{T\,\mathrm{OPT}}{\Phi_{\min}}+1\right)
\Bigl(
A T_0
+
\sqrt{A\,T\!\left[
\operatorname{Reg}^r(T)
+
\operatorname{Reg}^c(T)
+
\log(C\,T)
\right]}
\Bigr). 
\end{split}
\end{equation}

Therefore, we obtain the regret guarantee stated in
Theorem~\ref{theorem: regret} for M$^2$-CMAB in
Algorithm~\ref{alg:The M$^2$-CMAB Algorithm}.
\end{proof}

\begin{table}[t]
\centering
\caption{Key hyperparameters used in the scheduling framework (placeholders).}
\label{tab:hyper}
\renewcommand{\arraystretch}{1.2}
\rowcolors{2}{white}{gray!20}
\begin{tabular}{p{3cm}<{\centering} m{\linewidth-6cm}}
\toprule
\textbf{Hyperparameters} & \multicolumn{1}{c}{\textbf{Description}}  \\
\midrule

$P_{\text{local}}$ & Average power consumption of the local device (W), 600 \\
$\kappa$ & Electricity price coefficient (cost per Joule), 2.06e-8 \\

\bottomrule
\end{tabular}
\end{table}

\section{Supplement for Experiments}\label{appendix: supply experiments}

\subsection{Distribution of Reward, Latency, and Monetary Cost Across Individual Datasets.}
\label{app:distribution}
Figure~\ref{fig:six_distrbution} illustrates the distribution of reward, latency, and monetary cost across MLLM inference backends for the five individual datasets as well as the composite dataset.

\begin{figure}[t]
\centering      
\includegraphics[width=\columnwidth*2/3]{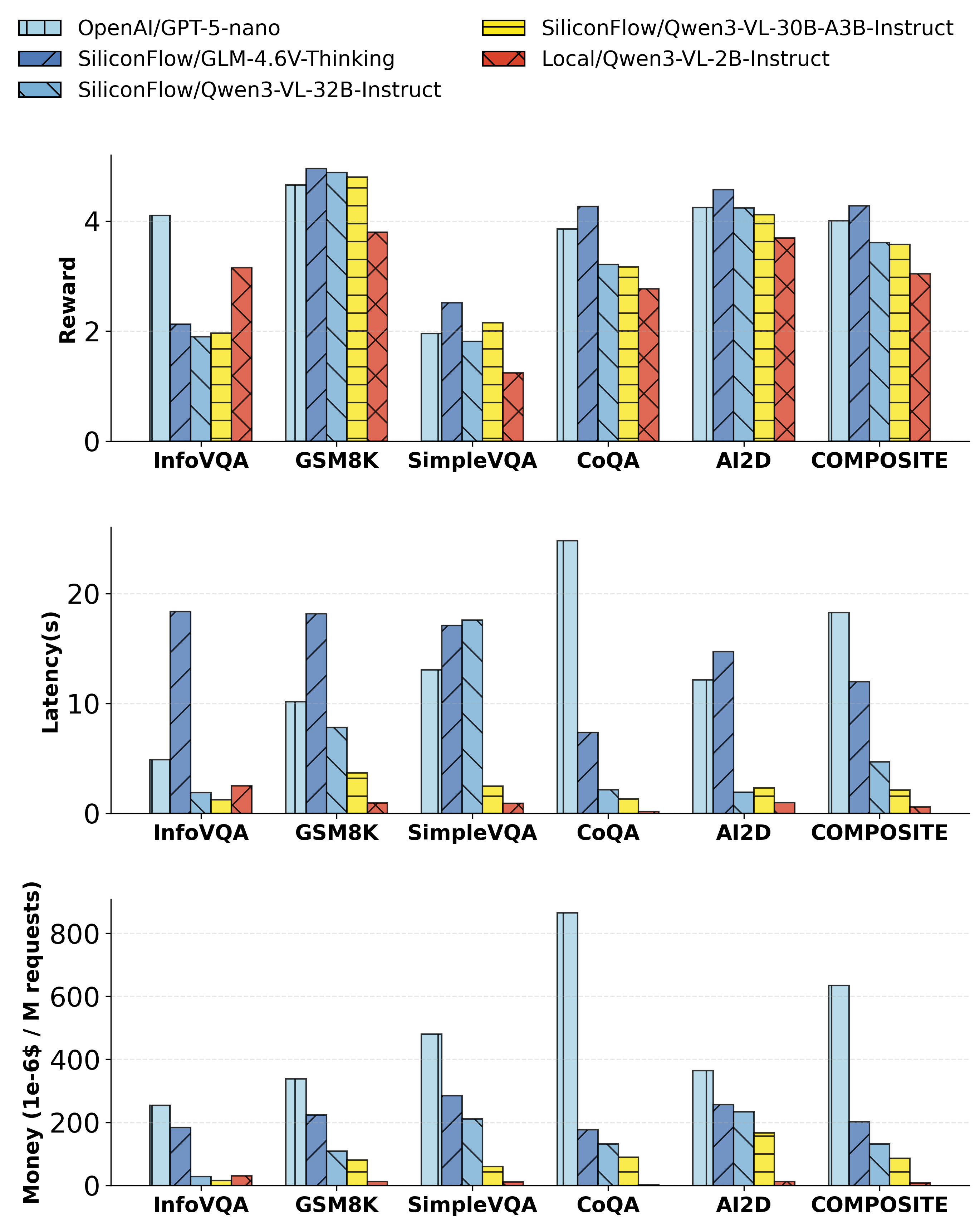}
\caption{Comparison of reward, latency, and monetary cost across MLLM inference backends on individual datasets and a mixed multi-modal task trace.}
\label{fig:six_distrbution}  
\vspace{-6mm}
\end{figure}



\subsection{Simulation Metrics and Reward Definitions}
\label{app:Metrics}

This appendix presents the formal mathematical definitions of the latency, monetary cost, and response quality (reward) metrics used throughout the experimental evaluation.

\paragraph{Local Execution.}
For local execution, inference latency is modeled as
\begin{equation}
\psi_{\text{Local}} =
\frac{T_{\text{in}} + T_{\text{out}}}{\text{FLOPS}_{\text{device}}},
\end{equation}
where $T_{\text{in}}$ and $T_{\text{out}}$ denote the numbers of input and output tokens, respectively, and $\text{FLOPS}_{\text{device}}$ is the effective compute throughput of the local device.
The corresponding monetary cost is defined as
\begin{equation}
\phi_{\text{Local}} =
\psi_{\text{Local}} \cdot P_{\text{local}} \cdot \kappa,
\end{equation}
where $P_{\text{local}}$ denotes the average power consumption of the device and $\kappa$ is the electricity price coefficient.

\paragraph{Cloud Execution.}
For cloud execution, inference latency $\psi_{\text{Cloud}}$ is modeled as the sum of network transmission delay and backend inference time, calibrated using empirical API traces.
The monetary cost of cloud execution is computed as
\begin{equation}
\phi_{\text{Cloud}} =
C_{\text{in}} + C_{\text{out}} + C_{\text{img}},
\end{equation}
where $C_{\text{in}}$ and $C_{\text{out}}$ denote provider-specific input and output token charges, and $C_{\text{img}}$ captures modality-dependent processing costs (e.g., image inputs).

\paragraph{Response Quality and Reward Definition.}
We define the response quality, denoted by $r \in [1,5]$, as the reward signal used to evaluate inference outcomes across different datasets.
Due to the heterogeneity of task formats and supervision levels, the reward is computed in a dataset-specific manner and normalized to ensure cross-dataset comparability. For standard generative QA tasks (e.g. CoQA), we compute the reward based on the similarity between the model-generated text and the reference answer. 
Specifically, we use the ROUGE-L~\cite{lin2004rouge} score and map it to an integer reward in $[1,5]$ according to Table~\ref{tab:reward-mapping}.
\begin{table}[t]
\centering
\caption{Mapping from ROUGE-L scores to discrete reward levels.}
\label{tab:reward-mapping}
\begin{tabular}{ll}
\toprule
\textbf{ROUGE-L} & \textbf{Reward}\\
\midrule
0
& 1 \\

(0, 0.15)
& 2\\

[0.15, 3)
& 3 \\

[0.3, 0.4)
& 4 \\

(0.4, 1]
& 5\\
\bottomrule
\end{tabular}
\end{table}

For multiple-choice tasks, where all candidate options (including the correct answer) are provided in the prompt (e.g. AI2D), we adopt exact-match evaluation. 
If the model’s selected answer matches the ground-truth option, a reward of 5 is assigned; otherwise, the reward is set to 1. 
This design ensures that the reward function is normalized and comparable across different task formats, facilitating consistent learning and scheduling within the M$^2$-CMAB framework.

\begin{table}[t]
\centering
\caption{Execution models used in the scheduling action space.}
\label{tab:exec-models}
\begin{tabular}{lll}
\toprule
\textbf{Category} & \textbf{Model} & \textbf{Provider / Deployment} \\
\midrule
Local
& Qwen3-VL-2B-Instruct
& On-device \\

Cloud
& GPT-5-nano
& OpenAI API \\

Cloud
& Qwen3-VL-32B-Instruct
& SiliconFlow API \\

Cloud
& Qwen3-VL-30B-A3B-Instruct
& On-device \\

Cloud
& GLM-4.6V-Thinking
& SiliconFlow API \\
\bottomrule
\end{tabular}
\end{table}

\subsection{Additional Implementation Details}
\label{app:impl}

We provide further details on baseline implementations, dataset preprocessing, and hardware specifications in this appendix to ensure full reproducibility. All experiments are conducted under identical workload traces and budget constraints.

Table~\ref{tab:hyper} summarizes the key hyperparameters involved in the scheduling and learning framework.
We report them here for completeness and reproducibility.
Unless otherwise specified, these hyperparameters are shared across all experiments.

\subsection{Execution Backend Configurations}
\label{app:exec-models}

Table~\ref{tab:exec-models} summarizes all execution backends considered in our action space, including their deployment location and provider.

\subsection{Prompt Used in Experiments}
\label{app:prompts}
\textbf{AI2D}
\begin{verbatim}
You are an assistant that answers questions based on a given image. Generate the 
answer in the format "Option x" (x is the correct option index/letter).
Image: {image}
Question: {question} 
Options: {options}
\end{verbatim}
\textbf{CoQA}
\begin{verbatim}
You are an assistant that answers questions based on a given story. Use only 
the information from the story to provide concise answers without adding extra 
content.
Backstory: {story}
Question: {question}
\end{verbatim}

\textbf{GSM8K}
\begin{verbatim}
You are a math expert. Please answer the given math problem. 
Generate the results strictly in the following format:  
"Analysis: Your Analysis Process. Answer: Your final answer to the question."
Question: {question}
\end{verbatim}

\textbf{InfoVQA \& SimpleVQA}
\begin{verbatim}
You are an assistant that answers questions based on a given image. 
Provide concise answers without adding any extra content. 
Image: {image}\n
Question: {question}
\end{verbatim}

%% file: ref.bib
@inproceedings{Wang2025,
  author       = {Jingyao Wang and
                  Siyu Zhao and
                  Wenwen Qiang and
                  Jiangmeng Li and
                  Changwen Zheng and
                  Fuchun Sun and
                  Hui Xiong},
  title        = {{Towards the Causal Complete Cause of Multi-Modal Representation Learning}},
  booktitle    = {Forty-second International Conference on Machine Learning (ICML 2025)}
                  ,
  publisher    = {PMLR},
  year         = {2025}
}

@inproceedings{Wu2025,
  author       = {Qilong Wu and
                  Yiyang Shao and
                  Jun Wang and
                  Xiaobo Sun},
  title        = {{Learning Optimal Multimodal Information Bottleneck Representations}},
  booktitle    = {Forty-second International Conference on Machine Learning (ICML 2025)}
                  ,
  publisher    = {PMLR},
  year         = {2025},

}

@inproceedings{Tu2025,
  author       = {Weijie Tu and
                  Weijian Deng and
                  Dylan Campbell and
                  Yu Yao and
                  Jiyang Zheng and
                  Tom Gedeon and
                  Tongliang Liu},
  title        = {Ranked from Within: Ranking Large Multimodal Models Without Labels},
  booktitle    = {Forty-second International Conference on Machine Learning (ICML 2025)}
                  ,
  publisher    = {PMLR},
  year         = {2025}
}

@inproceedings{Pei2025,
  author       = {Muleilan Pei and
                  Shaoshuai Shi and
                  Lu Zhang and
                  Peiliang Li and
                  Shaojie Shen},
  title        = {{GoIRL: Graph-Oriented Inverse Reinforcement Learning for Multimodal
                  Trajectory Prediction}},
  booktitle    = {Forty-second International Conference on Machine Learning (ICML 2025)}
                  ,
  publisher    = {PMLR},
  year         = {2025}
}

@inproceedings{Jia2025,
  author       = {Su Jia and
                  Peter I. Frazier and
                  Nathan Kallus},
  title        = {{Multi-Armed Bandits with Interference: Bridging Causal Inference and
                  Adversarial Bandits}},
  booktitle    = {Forty-second International Conference on Machine Learning (ICML 2025)}
                  ,
  publisher    = {PMLR},
  year         = {2025}
}

@inproceedings{Huang2025,
  author       = {Wenke Huang and
                  Jian Liang and
                  Zekun Shi and
                  Didi Zhu and
                  Guancheng Wan and
                  He Li and
                  Bo Du and
                  Dacheng Tao and
                  Mang Ye},
  title        = {{Learn from Downstream and Be Yourself in Multimodal Large Language
                  Models Fine-Tuning}},
  booktitle    = {Forty-second International Conference on Machine Learning (ICML 2025)}
                  ,
  publisher    = {PMLR},
  year         = {2025},
}

@inproceedings{Ma2025,
  author       = {Zehong Ma and
                  Shiliang Zhang and
                  Longhui Wei and
                  Qi Tian},
  title        = {{Efficient Multi-modal Long Context Learning for Training-free Adaptation}},
  booktitle    = {Forty-second International Conference on Machine Learning (ICML 2025)}
                  ,
  publisher    = {PMLR},
  year         = {2025},
}

@inproceedings{Ge2025,
  author       = {Chendi Ge and
                  Xin Wang and
                  Zeyang Zhang and
                  Hong Chen and
                  Jiapei Fan and
                  Longtao Huang and
                  Hui Xue and
                  Wenwu Zhu},
  title        = {Dynamic Mixture of Curriculum LoRA Experts for Continual Multimodal
                  Instruction Tuning},
  booktitle    = {Forty-second International Conference on Machine Learning (ICML 2025)}
                  ,
  publisher    = {PMLR},
  year         = {2025},
}

@inproceedings{Oh2025,
  author       = {Changdae Oh and
                  Zhen Fang and
                  Shawn Im and
                  Xuefeng Du and
                  Yixuan Li},
  title        = {{Understanding Multimodal LLMs Under Distribution Shifts: An Information-Theoretic
                  Approach}},
  booktitle    = {Forty-second International Conference on Machine Learning (ICML 2025)}
                  ,
  publisher    = {PMLR},
  year         = {2025}
}

@inproceedings{Chen2025,
  author       = {Mingcai Chen and
                  Baoming Zhang and
                  Zongbo Han and
                  Wenyu Jiang and
                  Yanmeng Wang and
                  Shuai Feng and
                  Yuntao Du and
                  Bingkun Bao},
  title        = {{Test-Time Selective Adaptation for Uni-Modal Distribution Shift in
                  Multi-Modal Data}},
  booktitle    = {Forty-second International Conference on Machine Learning (ICML 2025)}
                  ,
  publisher    = {PMLR},
  year         = {2025},
}

@inproceedings{Hao2025,
  author       = {Hao Fei and
                  Yuan Zhou and
                  Juncheng Li and
                  Xiangtai Li and
                  Qingshan Xu and
                  Bobo Li and
                  Shengqiong Wu and
                  Yaoting Wang and
                  Junbao Zhou and
                  Jiahao Meng and
                  Qingyu Shi and
                  Zhiyuan Zhou and
                  Liangtao Shi and
                  Minghe Gao and
                  Daoan Zhang and
                  Zhiqi Ge and
                  Siliang Tang and
                  Kaihang Pan and
                  Yaobo Ye and
                  Haobo Yuan and
                  Tao Zhang and
                  Weiming Wu and
                  Tianjie Ju and
                  Zixiang Meng and
                  Shilin Xu and
                  Liyu Jia and
                  Wentao Hu and
                  Meng Luo and
                  Jiebo Luo and
                  Tat{-}Seng Chua and
                  Shuicheng Yan and
                  Hanwang Zhang},
  title        = {{On Path to Multimodal Generalist: General-Level and General-Bench}},
  booktitle    = {Forty-second International Conference on Machine Learning (ICML 2025)}
                  ,
  publisher    = {PMLR},
  year         = {2025},
}

@misc{li2024llava_next_interleave,
  title         = {LLaVA-NeXT-Interleave: Tackling Multi-image, Video, and 3D in Large Multimodal Models},
  author        = {Li, Feng and Zhang, Renrui and Zhang, Hao and Zhang, Yuanhan and Li, Bo and Li, Wei and Ma, Zejun and Li, Chunyuan},
  year          = {2024},
  eprint        = {2407.07895},
  archivePrefix = {arXiv},
  primaryClass  = {cs.CV},
  doi           = {10.48550/arXiv.2407.07895},
  url           = {https://arxiv.org/abs/2407.07895}
}

@misc{li2024llava_onevision,
  title         = {LLaVA-OneVision: Easy Visual Task Transfer},
  author        = {Li, Bo and Zhang, Yuanhan and Guo, Dong and Zhang, Renrui and Li, Feng and Zhang, Hao and Zhang, Kaichen and Zhang, Peiyuan and Li, Yanwei and Liu, Ziwei and Li, Chunyuan},
  year          = {2024},
  eprint        = {2408.03326},
  archivePrefix = {arXiv},
  primaryClass  = {cs.CV},
  doi           = {10.48550/arXiv.2408.03326},
  url           = {https://arxiv.org/abs/2408.03326}
}

@misc{bai2025qwen3vl,
  title         = {Qwen3-VL Technical Report},
  author        = {Bai, Shuai and Cai, Yuxuan and Chen, Ruizhe and Chen, Keqin and Chen, Xionghui and others},
  year          = {2025},
  eprint        = {2511.21631},
  archivePrefix = {arXiv},
  primaryClass  = {cs.CV},
  doi           = {10.48550/arXiv.2511.21631},
  url           = {https://arxiv.org/abs/2511.21631}
}

@misc{bai2025qwen25vl,
  title         = {Qwen2.5-VL Technical Report},
  author        = {Bai, Shuai and Chen, Keqin and Liu, Xuejing and Wang, Jialin and Ge, Wenbin and others},
  year          = {2025},
  eprint        = {2502.13923},
  archivePrefix = {arXiv},
  primaryClass  = {cs.CV},
  doi           = {10.48550/arXiv.2502.13923},
  url           = {https://arxiv.org/abs/2502.13923}
}

@inproceedings{Liu2025,
  author       = {Xutong Liu and
                  Xiangxiang Dai and
                  Jinhang Zuo and
                  Siwei Wang and
                  Carlee Joe{-}Wong and
                  John C. S. Lui and
                  Wei Chen},
  title        = {{Offline Learning for Combinatorial Multi-armed Bandits}},
  booktitle    = {Forty-second International Conference on Machine Learning (ICML 2025)},
  publisher    = {PMLR},
  year         = {2025},
}

@ARTICLE{Li2023,
  author={Li, Weiting and Xiang, Liyao and Guo, Bin and Li, Zhetao and Wang, Xinbing},
  journal={IEEE Transactions on Information Forensics and Security}, 
  title={{DPlanner: A Privacy Budgeting System for Utility}}, 
  year={2023},
  volume={18},
  number={},
  pages={1196-1210}}

@ARTICLE{Zhang2026,
  author={Zhang, Xianzhi and Zhou, Yipeng and Wu, Di and Sheng, Quan Z. and Hu, Miao and Xiao, Linchang},
  journal={IEEE Transactions on Networking}, 
  title={{PPVF: An Efficient Privacy-Preserving Online Video Fetching Framework With Correlated Differential Privacy}}, 
  year={2026},
  volume={34},
  number={},
  pages={973-988}}

@inproceedings{DBLP:conf/aistats/HanZWXZ23,
  author       = {Yuxuan Han and
                  Jialin Zeng and
                  Yang Wang and
                  Yang Xiang and
                  Jiheng Zhang},
  title        = {{Optimal Contextual Bandits with Knapsacks under Realizability via
                  Regression Oracles}},
  booktitle    = {International Conference on AISTATS},
  pages        = {5011--5035},
  publisher    = {{PMLR}},
  year         = {2023},
  bibsource    = {dblp computer science bibliography, https://dblp.org}
}

@BOOK{DBLP:journals/ftopt/Hazan16,
  author={Hazan, Elad},
  title={Introduction to Online Convex Optimization},
  year={2016},
  volume={},
  number={},
  pages={},
  publisher ={Now Foundations and Trends},
  doi={10.1561/2400000013}}

@inproceedings{ouyang2023dynamic,
  title={{Dynamic Edge-centric Resource Provisioning for Online and Offline Services Co-location}},
  author={Ouyang, Tao and Zhao, Kongyange and Zhang, Xiaoxi and Zhou, Zhi and Chen, Xu},
  booktitle={IEEE Conference on Computer Communications (INFOCOM)},
  pages={1--10},
  year={2023},
  organization={IEEE}
}

@misc{yang2025moaoff,
      title={MoA-Off: Adaptive Heterogeneous Modality-Aware Offloading with Edge-Cloud Collaboration for Efficient Multimodal LLM Inference}, 
      author={Zheming Yang and Qi Guo and Yunqing Hu and Chang Zhao and Chang Zhang and Jian Zhao and Wen Ji},
      year={2025},
      eprint={2509.16995},
      archivePrefix={arXiv},
      primaryClass={cs.DC},
      url={https://arxiv.org/abs/2509.16995}, 
}

@INPROCEEDINGS{Zhou2025,
  author={Zhou, Chengjin and Liu, Song and Yuan, Xinjing and Shi, Jianxin and Zhang, Yuan and Pu, Lingjun},
  booktitle={{2025 IEEE/ACM 33rd International Symposium on Quality of Service (IWQoS)}}, 
  title={Libra: Novel LLM Token Streaming via Region-Based Task Scheduling and Token Bundling}, 
  year={2025},
  volume={},
  number={},
  pages={1-10}}

@misc{li2025CollaborativeSurvey,
      title={Collaborative Inference and Learning between Edge SLMs and Cloud LLMs: A Survey of Algorithms, Execution, and Open Challenges}, 
      author={Senyao Li and Haozhao Wang and Wenchao Xu and Rui Zhang and Song Guo and Jingling Yuan and Xian Zhong and Tianwei Zhang and Ruixuan Li},
      year={2025},
      eprint={2507.16731},
      archivePrefix={arXiv},
      primaryClass={cs.DC},
      url={https://arxiv.org/abs/2507.16731}, 
}

@inproceedings{Xia2024,
author = {Xia, Yu and Kong, Fang and Yu, Tong and Guo, Liya and Rossi, Ryan A. and Kim, Sungchul and Li, Shuai},
title = {{Which LLM to Play? Convergence-Aware Online Model Selection with Time-Increasing Bandits}},
year = {2024},
publisher = {ACM},
booktitle = {Proceedings of the ACM Web Conference 2024 (WWW '24)},
pages = {4059–4070},
numpages = {12}
}

@inproceedings{Agrawal2016,
author = {Agrawal, Shipra and Devanur, Nikhil R.},
title = {{Linear contextual bandits with knapsacks}},
year = {2016},
publisher = {Curran Associates Inc.},
booktitle = {Proceedings of the 30th International Conference on Neural Information Processing Systems (NIPS'16)},
pages = {3458–3467},
}

@article{Shai2012,
  author       = {Shai Shalev{-}Shwartz},
  title        = {{Online Learning and Online Convex Optimization}},
  journal      = {Foundations and Trends in Machine Learning},
  volume       = {4},
  number       = {2},
  pages        = {107--194},
  year         = {2012},
}

@inproceedings{Yuan2025,
  author       = {Liangqi Yuan and
                  Dong{-}Jun Han and
                  Shiqiang Wang and
                  Christopher Brinton},
  title        = {{Local-Cloud Inference Offloading for LLMs in Multi-Modal, Multi-Task,
                  Multi-Dialogue Settings}},
  booktitle    = {Proceedings of the Twenty-sixth International Symposium on Theory,
                  Algorithmic Foundations, and Protocol Design for Mobile Networks and
                  Mobile Computing (MobiHoc'25)},
  pages        = {201--210},
  year         = {2025},
    publisher = {ACM}
}

@article{Li2025,
  author       = {Yang Li and
                  Xianzhi Zhang and
                  Miao Hu and
                  Di Wu and
                  Yipeng Zhou},
  title        = {{{W2CB:} Online Data Acquisition Optimization With Wasserstein Contextual
                  Combinatorial Bandits}},
  journal      = {IEEE Transactions on Services Computing},
  volume       = {18},
  number       = {6},
  pages        = {4349--4363},
  year         = {2025},
}

@article{Zhang2025TSC,
  author       = {Xianzhi Zhang and
                  Yipeng Zhou and
                  Miao Hu and
                  Di Wu and
                  Pengshan Liao and
                  Mohsen Guizani and
                  Quan Z. Sheng},
  title        = {{BGTplanner: Maximizing Training Accuracy for Differentially Private
                  Federated Recommenders via Strategic Privacy Budget Allocation}},
  journal      = {IEEE Transactions on Services Computing},
  volume       = {18},
  number       = {6},
  pages        = {3523--3537},
  year         = {2025},
}

@inproceedings{Fu2024NIPS,
author = {Fu, Yichao and Zhu, Siqi and Su, Runlong and Qiao, Aurick and Stoica, Ion and Zhang, Hao},
title = {{Efficient LLM scheduling by learning to rank}},
year = {2024},
publisher = {Curran Associates Inc.},
booktitle = {Proceedings of the 38th International Conference on Neural Information Processing Systems (NIPS '24)}
}

@misc{liyang2025arXiv,
      title={{LLM Bandit: Cost-Efficient LLM Generation via Preference-Conditioned Dynamic Routing}}, 
      author={Yang Li},
      year={2025},
      eprint={2502.02743},
      archivePrefix={arXiv},
      primaryClass={cs.LG},
      url={https://arxiv.org/abs/2502.02743}, 
}

@misc{yang2024PerLLMArXiv,
      title={{PerLLM: Personalized Inference Scheduling with Edge-Cloud Collaboration for Diverse LLM Services}}, 
      author={Zheming Yang and Yuanhao Yang and Chang Zhao and Qi Guo and Wenkai He and Wen Ji},
      year={2024},
      eprint={2405.14636},
      archivePrefix={arXiv},
      primaryClass={cs.DC},
      url={https://arxiv.org/abs/2405.14636}, 
}

@misc{li2025arXiv,
      title={{Collaborative Inference and Learning between Edge SLMs and Cloud LLMs: A Survey of Algorithms, Execution, and Open Challenges}}, 
      author={Senyao Li and Haozhao Wang and Wenchao Xu and Rui Zhang and Song Guo and Jingling Yuan and Xian Zhong and Tianwei Zhang and Ruixuan Li},
      year={2025},
      eprint={2507.16731},
      archivePrefix={arXiv},
      primaryClass={cs.DC},
      url={https://arxiv.org/abs/2507.16731}, 
}

@InProceedings{Xu_2025_ICCV,
    author    = {Xu, Zhuoyan and Nguyen, Khoi Duc and Mukherjee, Preeti and Bagchi, Saurabh and Chaterji, Somali and Liang, Yingyu and Li, Yin},
    title     = {Learning to Inference Adaptively for Multimodal Large Language Models},
    booktitle = {Proceedings of the IEEE/CVF International Conference on Computer Vision (ICCV)},
    publisher = {IEEE/CVF},
    year      = {2025},
    pages     = {3552-3563}
}

@inproceedings{Wu2015,
author = {Wu, Huasen and Srikant, R. and Liu, Xin and Jiang, Chong},
title = {{Algorithms with logarithmic or sublinear regret for constrained contextual bandits}},
year = {2015},
publisher = {ACM},
booktitle = {Proceedings of International Conference on NeurIPS},
pages = {433–441},
numpages = {9},
}

@misc{cheng2025simplevqa,
      title={SimpleVQA: Multimodal Factuality Evaluation for Multimodal Large Language Models}, 
      author={Xianfu Cheng and Wei Zhang and Shiwei Zhang and Jian Yang and Xiangyuan Guan and Xianjie Wu and Xiang Li and Ge Zhang and Jiaheng Liu and Yuying Mai and Yutao Zeng and Zhoufutu Wen and Ke Jin and Baorui Wang and Weixiao Zhou and Yunhong Lu and Tongliang Li and Wenhao Huang and Zhoujun Li},
      year={2025},
      eprint={2502.13059},
      archivePrefix={arXiv},
      primaryClass={cs.CL},
      url={https://arxiv.org/abs/2502.13059}, 
}

@misc{cobbe2021gsm8k,
      title={Training Verifiers to Solve Math Word Problems}, 
      author={Karl Cobbe and Vineet Kosaraju and Mohammad Bavarian and Mark Chen and Heewoo Jun and Lukasz Kaiser and Matthias Plappert and Jerry Tworek and Jacob Hilton and Reiichiro Nakano and Christopher Hesse and John Schulman},
      year={2021},
      eprint={2110.14168},
      archivePrefix={arXiv},
      primaryClass={cs.LG},
      url={https://arxiv.org/abs/2110.14168}, 
}

@misc{kembhavi2016diagram,
      title={A Diagram Is Worth A Dozen Images}, 
      author={Aniruddha Kembhavi and Mike Salvato and Eric Kolve and Minjoon Seo and Hannaneh Hajishirzi and Ali Farhadi},
      year={2016},
      eprint={1603.07396},
      archivePrefix={arXiv},
      primaryClass={cs.CV},
      url={https://arxiv.org/abs/1603.07396}, 
}

@article{reddy-etal-2019-coqa,
    title = {{{C}o{QA}: A Conversational Question Answering Challenge}},
    author = "Reddy, Siva  and
      Chen, Danqi  and
      Manning, Christopher D.",
    journal = "Transactions of the Association for Computational Linguistics",
    volume = "7",
    year = "2019",
    publisher = "MIT Press",
    pages = "249--266",
}

@misc{mathew2021infographicvqa,
      title={InfographicVQA}, 
      author={Minesh Mathew and Viraj Bagal and Rubèn Pérez Tito and Dimosthenis Karatzas and Ernest Valveny and C. V Jawahar},
      year={2021},
      eprint={2104.12756},
      archivePrefix={arXiv},
      primaryClass={cs.CV},
      url={https://arxiv.org/abs/2104.12756}, 
}

@inproceedings{lin2004rouge,
  title={Rouge: A package for automatic evaluation of summaries},
  author={Lin, Chin-Yew},
  booktitle={Text summarization branches out},
  pages={74--81},
  year={2004}
}
